\documentclass{article}
\usepackage[accepted]{icml2025}
\usepackage{pgfplots}
\usepackage[utf8]{inputenc}
\usepackage{microtype}

\usepackage[
  style=authoryear,
  maxcitenames=2,
  minnames=1,
  maxbibnames=4,
  sorting=nyt,
  giveninits=false,
  uniquelist=false,
  uniquename=false,
  dashed=false,
  doi=false,
  url=false,
  backref=false,
]{biblatex}
\usepackage[
  colorlinks=true,
  linkcolor=blue,
  filecolor=magenta,
  urlcolor=cyan,
  citecolor=purple,
]{hyperref}
\usepackage{fontawesome5}
\usepackage{xstring}

\newbibmacro*{doi-link}{\href{https://doi.org/\thefield{doi}}{\raisebox{0.067ex}{\small\faIcon[regular]{file-alt}}}}

\newbibmacro*{url-link}{\href{\thefield{url}}{\raisebox{0.067ex}{\small\faIcon[regular]{external-link-alt}}}}

\newbibmacro*{arxiv-link}{\href{https://arxiv.org/abs/\thefield{eprint}}{\raisebox{0.067ex}{\small\faIcon[regular]{external-link-alt}}}}

\newbibmacro*{biorxiv-link}{\href{https://biorxiv.org/content/\thefield{eprint}}{\raisebox{0.067ex}{\small\faIcon[regular]{external-link-alt}}}}

\newbibmacro*{psyarxiv-link}{\href{https://psyarxiv.com/\thefield{eprint}}{\raisebox{0.067ex}{\small\faIcon[regular]{external-link-alt}}}}

\renewbibmacro*{finentry}{\hspace{1ex}
  \iffieldundef{eprint}{\iffieldundef{doi}{\iffieldundef{url}{}{\usebibmacro{url-link}}}{\usebibmacro{doi-link}}}{\iffieldequalstr{eprinttype}{arXiv}{\usebibmacro{arxiv-link}}{\iffieldequalstr{eprinttype}{bioRxiv}{\usebibmacro{biorxiv-link}}{\iffieldequalstr{eprinttype}{PsyArXiv}{\usebibmacro{psyarxiv-link}}{\iffieldundef{doi}{\iffieldundef{url}{}{\usebibmacro{url-link}}}{\usebibmacro{doi-link}}}}}}}

\DeclareFieldFormat{doi}{}
\DeclareFieldFormat{url}{}
\AtEveryBibitem{
  \clearfield{eprint}
  \clearfield{eprinttype}
  \clearfield{eprintclass}
}
\AtEveryBibitem{\clearname{editor}}

\usepackage{amsmath}
\usepackage{amssymb}
\usepackage{mathtools}
\usepackage{amsthm}

\usepackage{enumitem}
\setlist[itemize,enumerate]{
    topsep=2pt,        partopsep=0pt,     leftmargin=*,
    parsep=0pt,        itemsep=2pt        }

\usepackage{graphicx}
\usepackage[nameinlink,capitalise]{cleveref}
\crefname{figure}{Figure}{Figures}
\crefname{equation}{Equation}{Equations}
\crefname{assumption}{Assumption}{Assumptions}
\usepackage[font=small,labelfont=bf]{caption}
\usepackage[format=hang]{subcaption}          

\usepackage{acro}
\DeclareAcronym{svd}{
  short=SVD,
  long=singular value decomposition
}

\DeclareAcronym{csvd}{
  short=cSVD,
  long=compact singular value decomposition
}

\DeclareAcronym{rsm}{
  short=RSM,
  long=representational similarity matrix
}

\DeclareAcronym{iid}{
  short=i.i.d.,
  long=independent and identically distributed
}

\DeclareAcronym{ann}{
  short=ANN,
  long=artificial neural network
}

\DeclareAcronym{relu}{
  short=ReLU,
  long=rectified linear unit
}

\DeclareAcronym{mds}{
  short=MDS,
  long=multidimensional scaling
}

\DeclareAcronym{rsa}{
  short=RSA,
  long=representational similarity analysis
}

\DeclareAcronym{gls}{
  short=GLS,
  long=general linear solution
}

\DeclareAcronym{lss}{
  short=LSS,
  long=least-squares solution
}

\DeclareAcronym{mrns}{
  short=MRNS,
  long=minimum representation-norm solution
}

\DeclareAcronym{mwns}{
  short=MWNS,
  long=minimum weight-norm solution
}

\theoremstyle{plain}
\newtheorem{theorem}{Theorem}[section]

\newtheorem{proposition}[theorem]{Proposition}

\newtheorem{corollary}[theorem]{Corollary}
\theoremstyle{definition}
\newtheorem{definition}[theorem]{Definition}
\newtheorem{assumption}[theorem]{Assumption}
\theoremstyle{remark}

\DeclareMathOperator{\rank}{rank}
\DeclareMathOperator{\image}{im}
\DeclareMathOperator{\svd}{cSVD}
\DeclareMathOperator{\rsm}{RSM}

\DeclareMathOperator*{\argmin}{argmin}
\DeclareMathOperator{\tr}{Tr}

\newcommand{\la}{\left\langle}
\newcommand{\ra}{\right\rangle}

\newcommand{\bfa}{\mathbf{a}}
\newcommand{\bfb}{\mathbf{b}}

\newcommand{\bfh}{\mathbf{h}}

\newcommand{\bfx}{\mathbf{x}}
\newcommand{\bfy}{\mathbf{y}}

\newcommand{\bfyhat}{\mathbf{\hat{y}}}
\newcommand{\bfxt}{\mathbf{\tilde{x}}}
\newcommand{\bfyt}{\mathbf{\tilde{y}}}

\newcommand{\bflambda}{\boldsymbol{\lambda}}
\newcommand{\bfxi}{\boldsymbol{\xi}}

\newcommand{\bfA}{\mathbf{A}}
\newcommand{\bfB}{\mathbf{B}}
\newcommand{\bfC}{\mathbf{C}}
\newcommand{\bfD}{\mathbf{D}}
\newcommand{\bfE}{\mathbf{E}}
\newcommand{\bfF}{\mathbf{F}}
\newcommand{\bfG}{\mathbf{G}}
\newcommand{\bfH}{\mathbf{H}}
\newcommand{\bfI}{\mathbf{I}}
\newcommand{\bfJ}{\mathbf{J}}
\newcommand{\bfK}{\mathbf{K}}
\newcommand{\bfL}{\mathbf{L}}
\newcommand{\bfM}{\mathbf{M}}
\newcommand{\bfN}{\mathbf{N}}
\newcommand{\bfO}{\mathbf{O}}
\newcommand{\bfP}{\mathbf{P}}
\newcommand{\bfQ}{\mathbf{Q}}
\newcommand{\bfR}{\mathbf{R}}
\newcommand{\bfS}{\mathbf{S}}

\newcommand{\bfU}{\mathbf{U}}

\newcommand{\bfV}{\mathbf{V}}

\newcommand{\bfW}{\mathbf{W}}
\newcommand{\bfX}{\mathbf{X}}
\newcommand{\bfY}{\mathbf{Y}}
\newcommand{\bfZ}{\mathbf{Z}}

\newcommand{\wa}{\bfW_1}
\newcommand{\wb}{\bfW_2}

\newcommand{\bfZero}{\textbf{0}}
\newcommand{\sigxx}{\boldsymbol{\Sigma}_{xx}}
\newcommand{\sigyy}{\boldsymbol{\Sigma}_{yy}}
\newcommand{\sigyx}{\boldsymbol{\Sigma}_{yx}}

\newcommand{\sigxxt}{\boldsymbol{\tilde{\Sigma}}_{xx}}

\newcommand{\sigyxt}{\boldsymbol{\tilde{\Sigma}}_{yx}}
\newcommand{\bfXt}{\mathbf{\tilde{X}}}
\newcommand{\bfYt}{\mathbf{\tilde{Y}}}

\newcommand{\bfGamma}{\boldsymbol{\Gamma}}
\newcommand{\bfPsi}{\boldsymbol{\Psi}}
\newcommand{\bfPhi}{\boldsymbol{\Phi}}
\newcommand{\bfLambda}{\boldsymbol{\Lambda}}
\newcommand{\bfXi}{\boldsymbol{\Xi}}

\newcommand{\wbar}{\mathbf{\bar{W}}}
\newcommand{\wobar}{\mathbf{\bar{\Omega}}}
\newcommand{\wao}{\mathbf{\Omega}_1}
\newcommand{\wbo}{\mathbf{\Omega}_2}

\newcommand{\fa}{\textbf{(A)}~}
\newcommand{\fb}{\textbf{(B)}~}
\newcommand{\fc}{\textbf{(C)}~}
\newcommand{\fd}{\textbf{(D)}~}
\newcommand{\fe}{\textbf{(E)}~}

\usepackage{marginnote}
\usepackage{xcolor}
\newcommand{\ignore}[1]{}

\definecolor{electric-purple}{RGB}{191, 0, 255}
\newcommand{\erin}[1]{\marginnote{\tiny\textsf{\color{electric-purple}[#1]}}}
\definecolor{complementary-green}{RGB}{0, 255, 191}
\newcommand{\lukas}[1]{\marginnote{\tiny\textsf{\color{complementary-green}[#1]}}}
\definecolor{complementary-orange}{RGB}{255, 191, 0}
\newcommand{\andrew}[1]{\marginnote{\tiny\textsf{\color{complementary-orange}[#1]}}}

\renewcommand{\erin}[1]{}
\renewcommand{\lukas}[1]{}
\renewcommand{\andrew}[1]{}

\newcommand{\eg}{\textit{e.g.},~}
\newcommand{\ie}{\textit{i.e.},~}

\icmltitlerunning{An analytical dissociation of functional and representational similarity in deep linear neural networks}

\begin{document}

\twocolumn[
  \icmltitle{Not all solutions are created equal: An analytical dissociation of\\functional and representational similarity in deep linear neural networks}

  \begin{icmlauthorlist}
    \icmlauthor{Lukas Braun}{expsy}
    \icmlauthor{Erin Grant}{gatsby}
    \icmlauthor{Andrew M. Saxe}{gatsby}
  \end{icmlauthorlist}

  \icmlaffiliation{expsy}{Department of Experimental Psychology, University of Oxford, Oxford, UK}
  \icmlaffiliation{gatsby}{Gatsby Unit \& Sainsbury Wellcome Centre, University College London, London, UK}

  \icmlcorrespondingauthor{Lukas Braun}{lukas.braun@psy.ox.ac.uk}

  \icmlkeywords{deep learning, representational similarity, representation learning, deep linear networks, neuroscience}

  \vskip 0.3in
]

\printAffiliationsAndNotice{}

\begin{abstract}

  A foundational principle of connectionism is that perception, action, and cognition emerge from parallel computations among simple, interconnected units that generate and rely on neural representations.
  Accordingly, researchers employ multivariate pattern analysis to decode and compare the neural codes of artificial and biological networks, aiming to uncover their functions.
  However, there is limited analytical understanding of how a network’s representation and function relate, despite this being essential to any quantitative notion of underlying function or functional similarity.
  We address this question using analysable two-layer linear networks and numerical simulations in nonlinear networks.
  We find that function and representation are dissociated, allowing representational similarity without functional similarity and vice versa.
  Further, we show that neither robustness to input noise nor the level of generalization error constrain representations to the task.
  In contrast, networks robust to parameter noise have limited representational flexibility and must employ task-specific representations.
  Our findings suggest that representational alignment reflects computational advantages beyond functional alignment alone, with significant implications for interpreting and comparing the representations of connectionist systems.
\end{abstract}

\section{Introduction}
The \emph{parallel distributed processing} hypothesis posits that function in artificial and biological networks emerges from interactions among simple interconnected units that compute with distributed representations~\parencite{rumelhart1986parallel}.
Accordingly, one might aim to identify function from networks observables such as connectivity weights and neural activity patterns; however, this is often complicated by the inherent complexity and partial observability of these systems.
In particular, the structure of artificial and biological networks is often \emph{non-identifiable} in the sense that networks can be structurally distinct, yet implement the same input-output mapping.
\begin{figure}[t]
  \centering
  \includegraphics[width=\columnwidth]{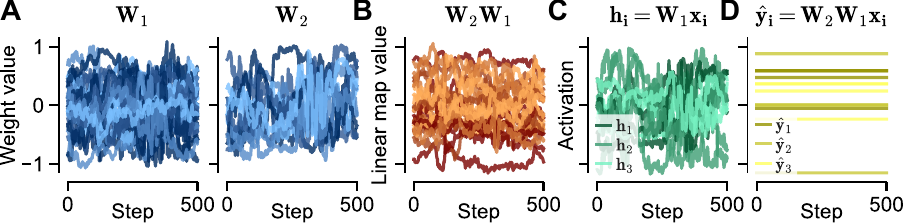}
  \caption{
    \textbf{Random walk}.
    \fa A random walk on the solution manifold of a two-layer linear network reveals that input and readout weights can change continuously, inducing changes in the \fb network parametrisation and thus the \fc hidden-layer representations, while preserving the \fd network output.
  }
  \label{fig:intro}
  \vspace{-4ex}
\end{figure}
For example, biophysical neuron models can exhibit nearly identical functions at both the neuron~\parencite{goldman2001global} and network levels~\parencite{prinz2004similar} despite considerable variation in their architecture~\parencite[reviewed in][]{marder2006variability, albantakis2024brain}.
Similarly, \acp{ann} are almost always non-identifiable due to simple symmetries, such as permutation-invariance of neurons~\parencite{sussmann1992uniqueness, albertini1993neural}, scale-invariance of activation functions~\parencite{neyshabur2015pathsgd}, alongside more complex symmetries arising from feature composition across layers and from finite training data~\parencite{refinetti2021classifying, arous2023highdimensional}.
As modern networks are deep and heavily overparametrised~\parencite{10.1145/3446776}, they are inherently non-identifiable, with many parametrisations yielding the same input-output behaviour.
Determining when parametrisations become identifiable and understanding the consequences of non-identifiability remain open problems~\parencite{roeder2021linear, entezari2022role, vlacic2022neural, ghosh2022pitfalls, wang2022desiderata, godfrey2022symmetries, martinelli2023expandandcluster, bona-pellissier2023parameter, lampinen2024learned, kori2024identifiable, marconato2024all}.

Even deep linear networks exhibit both trivial and non-trivial symmetries, making their parametrisation non-identifiable.
While any deep linear network can be re-parametrised as a single linear transformation~\parencite{laurent2018deep}, it does so through multistage computations that give rise to hidden-layer representations.
Moreover, the optimisation landscape of a deep linear network is non-convex and contains a high-dimensional solution manifold~(\cref{fig:intro}) whose shape is determined by the statistics of training data and the network architecture~\parencite{baldi1989neural, saxe2014exact, arora2019convergence}, making it a useful surrogate for studying representation learning~\parencite{saxe2019mathematical, braun2022exact, domine2024lazy}.
Here, we leverage the analytical tractability of deep linear networks to study functionally equivalent parametrisations at global minimum error.
Crucially, these solutions employ different internal representations, which has significant computational consequences, most notably in their affordances for linear decoding, representational similarity analysis (\cref{sec:implications}), and their sensitivity to noise (\cref{sec:advantages}).

We now detail our \textbf{contributions:}
\begin{itemize}
  \item We derive exact parametric equations characterising the complete and distinct subregions of the solution manifold in two-layer linear networks.
  \item We demonstrate that, although all subregions allow flexible neural representations, some inherently lead to identifiable task-specific representational similarities, while others result in non-identifiable, task-agnostic representational similarities.
  \item We establish that in contrast to task-specific solutions, task-agnostic solutions are non-identifiable and non-comparable.
  \item We analytically show that input noise and generalisation error do not constrain representations to task-specific regions, whereas parameter noise does.
  \item We validate our analytical findings through numerical simulations, demonstrating that these computational principles persist in non-linear neural networks.
\end{itemize}

All simulations are detailed in \cref{app:sec:simulation-details}, and a code repository reproducing all figures is available on GitHub at \href{https://github.com/lukas-braun/dissociating-similarity}{\texttt{lukas-braun/dissociating-similarity}}.
\section{Setting and preliminaries}
We consider a two-layer linear network~(\cref{fig:intro}A),
\begin{equation}
  \bfyhat_n = \wb\wa\bfx_n,
\end{equation}
trained to minimise the mean-squared error
\begin{equation}
  \mathcal{L}_\text{MSE} = \frac{1}{2P}\sum_{n=1}^P||\bfyhat_n - \bfy_n||_2^2
\end{equation}
over a dataset $\mathcal{D}=\{(\bfx_n, \bfy_n)\}_{n=1}^P$, with inputs $\bfx_n \in \mathbb{R}^{N_i}$ and corresponding targets $\bfy_n \in \mathbb{R}^{N_o}$.
The input weights $\wa \in \mathbb{R}^{N_h \times N_i}$ project inputs to hidden-layer neural representation $\bfh_n = \wa\bfx_n \in \mathbb{R}^{N_h}$, which are projected to outputs via the readout weights $\wb \in \mathbb{R}^{N_o \times N_h}$.
We denote by $\bfX = [\bfx_1,.
      .., \bfx_P]$, $\bfY = [\bfy_1,..., \bfy_P]$, and $\bfH = [\bfh_1,..., \bfh_P]$ the matrices that contain all inputs, targets and hidden-layer representations, respectively.
The network's \ac{rsm} is then defined by
\begin{equation}
  \rsm = \bfX^T\wa^T\wa\bfX = \bfH^T\bfH,
\end{equation}
capturing pairwise similarities between inputs in the hidden representational space. \textcite{laurent2018deep} showed that, under the following assumptions:
\begin{assumption} \label{ass:convex-differentiable-loss}
  The loss function is convex and differentiable, \eg the mean-squared error loss.
\end{assumption}
\begin{assumption} \label{ass:no-bottleneck}
  The network is not bottle-necked, \ie $\min\left(N_i, N_o\right) \leq N_h$ \end{assumption} \noindent all local minima of a deep linear network are global and equivalent to the solution of the corresponding single-layer linear regression problem.
Notably, this result holds without assumptions on the structure of the training data.
In our setting, this permits the following definition (see \cref{app:sec:general-linear-solution}):
\begin{definition} \label{def:general-linear-solution}
  Under \cref{ass:convex-differentiable-loss,ass:no-bottleneck}, any pair of network weights satisfying
  \begin{equation} \label{eq:general-linear-solution}
    \wb\wa\sigxx = \sigyx
  \end{equation}
  is a globally optimal \ac{gls}, where
  \begin{equation}
    \sigxx = \frac{1}{P}\sum_{n=1}^P \bfx_n\bfx_n^T\quad \text{and} \quad \sigyx = \frac{1}{P}\sum_{n=1}^P \bfy_n\bfx_n^T
  \end{equation}
  denote the input and input-output covariance matrices.
\end{definition}
Note that the absence of suboptimal minima does not preclude the existence of other critical points, nor does it guarantee convergence of gradient-based algorithms.
To distinguish general weight matrices $\wa$ and $\wb$ from those that satisfy \cref{eq:general-linear-solution}, we denote the latter as optimal weights $\wao$ and $\wbo$.
Further, we denote the \ac{csvd}, as defined in \cref{app:sec:csvd}, of the inputs and least-squares solution as
\begin{equation}
  \svd(\bfX) = \bfA\bfB\bfC^T,
\end{equation}
and
\begin{equation}
  \svd(\sigyx(\sigxx)^+) = \bfU\bfS\bfV^T.
\end{equation}
In the following, we analytically study the full set of global solutions, the so called solution manifold, and partition it into subregions with distinct representational and computational properties.
Crucially, our analysis holds without assumptions on the structure of the training data and irrespective of how a solution is obtained, and thus does not depend on any particular learning or optimisation algorithm.
\Cref{fig:manifold}A visualises the entire solution manifold and its subregions for a simple example, providing some intuition for the formal definitions and theorems developed next.
\section{Partitioning the solution manifold of two-layer linear networks}
\label{sec:manifold}
\begin{figure*}[!t]
  \centering
  \includegraphics[width=\textwidth]{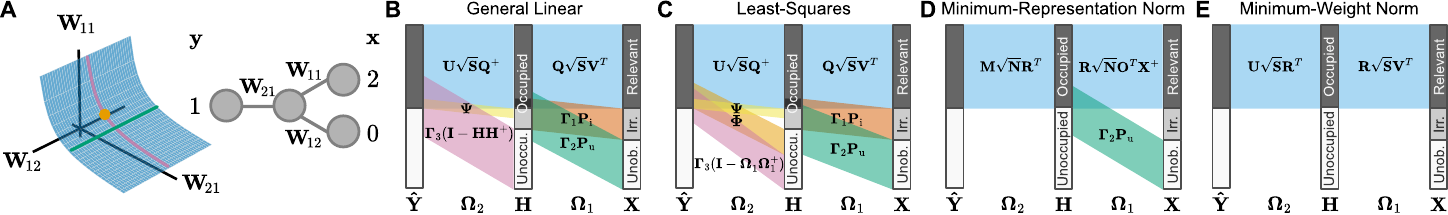}
  \caption{
    \textbf{Solution manifold}.
    \fa Schematic of solution manifold (left) for a two-layer linear network trained on a single training pair (right).
    The \ac{gls} (blue plane) reflects that weight $\mathbf{W}_{12}$ lies in the input null space and is unconstrained, while $\mathbf{W}_{11}$ and $\mathbf{W}_{21}$ are coupled, an increase in one requires a decrease in the other.
    Constrained LSS (pink), MRNS (green), MWNS (orange) are highlighted subregions of the manifold.
    \fb Schematic of the parametrisation of the GLS, showing how components of $\wao$ map relevant, irrelevant, and unobserved input directions to the hidden space, and how components of $\wbo$ map from unoccupied and occupied hidden directions to the output.
    Projections from irrelevant inputs can interfere with the core (blue), creating overlap (black) between the relevant (dark grey) and irrelevant (light grey) hidden space , which is cancelled by $\bfPsi$.
    \fc As in \fb, but for LSS.
    Projections from unobserved input directions into the occupied hidden space are cancelled by $\bfPhi$.
    \fd and \fe are as in \fb, but show MRNS and MWNS, respectively.
    The additional constraints remove projections and further restrict the core.
  }
  \label{fig:manifold}
  \vspace{-0.3cm}
\end{figure*}
Two-layer linear networks architectures are typically highly overparametrised, admitting many combinations of input and output weights that achieve the global optimum for a given task.
Formally, the set of all such network weights defines the solution manifold,
\begin{equation}
  \mathcal{M} = \Big\{\wbo\wao  : \wbo\wao\sigxx = \sigyx\Big\}~.
\end{equation}
We note that useful intuition about the manifold's structure can be gained by viewing it as the set of weight configurations related by invertible linear transformations.
For any invertible matrix $\bfQ \in \mathcal{R}^{N_h \times N_h}$, the weight pair
\begin{equation}
  \wbo \rightarrow \wbo\bfQ^{-1}\quad \text{and}\quad \wao \rightarrow \bfQ\wao
\end{equation}
implements the same input-output map and thus lies on the same manifold \parencite{baldi1989neural, saxe2014exact}.
In the context of a neural network architecture, these $\bfQ$-transformations redistribute how information is processed across layers, for example, by rotating and scaling intermediate representations.
However, since they preserve rank, they only fully characterise the solution manifold when the input and output dimensions are identical and the task has full rank, conditions that may not be met in real-world scenarios.
To refine this view, we partition the input space into three subspaces: \emph{Observed and relevant}, \emph{observed but irrelevant}, and \emph{unobserved} null directions, with corresponding projections $\bfP_\text{r} = \bfV\bfV^T$, $\bfP_\text{i} = \bfA\bfA^T - \bfV\bfV^T$, and $\bfP_\text{u} = \bfI - \bfA\bfA^T$.
The distinction between relevant and irrelevant directions arises because the input space can exceed the intrinsic dimensionality of the solution manifold (but not vice versa).
Specifically,
\begin{equation}
  r = \rank(\sigyx) \leq \min(\rank(\bfX), \rank(\bfY)),
\end{equation}
so when the input rank exceeds the target rank, some input directions are irrelevant to solving the task.
Likewise, the hidden space can be partitioned into subspaces corresponding to the hidden-layer representations of relevant, and irrelevant inputs, and all remaining unoccupied null directions.
Importantly, the hidden representations of relevant and irrelevant inputs may overlap, which necessitates compensation and introduces structural constraints on the form of valid solutions.
The following parametrized equation fully encapsulates this intricate structure of the solution manifold:
\begin{theorem} \label{the:general-linear-solution}
  Any \ac{gls} satisfies
  \begin{equation}
    \begin{aligned}
      \wao & = \bfQ\sqrt{\bfS}\bfV^T + \bfGamma_1\bfP_\text{i} + \bfGamma_2\bfP_\text{u}\ \ \text{and} \\
      \wbo & = \bfU\sqrt{\bfS}\bfQ^+ + \bfPsi + \bfGamma_3(\bfI - \bfH\bfH^+),
    \end{aligned}
  \end{equation}
  where $\bfQ \in \mathcal{R}^{N_h \times r}$ is an arbitrary full-column-rank matrix, $\bfGamma_1$, $\bfGamma_2 \in \mathcal{R}^{N_h \times N_i}$ are arbitrary matrices subject to the constraint $\rank(\bfQ\bfQ^+\bfGamma_1\bfP_\text{i}) \leq \rank((\bfI - \bfQ\bfQ^+)\bfGamma_1\bfP_\text{i})$, $\bfPsi = -\bfU\sqrt{\bfS}\bfQ^+\bfGamma_1\bfP_\text{i}\left[(\bfI - \bfQ\bfQ^+)\bfGamma_1\bfP_\text{i}\right]^+$, and $\bfGamma_3 \in \mathcal{R}^{N_o \times N_h}$ is an arbitrary matrix.
\end{theorem}
The first terms of $\wao$ and $\wbo$ implement the core input-output mapping; $\bfGamma_1$ and $\bfGamma_2$ project from task-irrelevant and unobserved input directions; $\bfGamma_3$ projects from the unoccupied hidden space; $\bfPsi$ cancels interference from irrelevant inputs that are projected into the core; and the rank constraint ensures that such a correction exists.
See \cref{app:sec:general-linear-solution} for a detailed proof and \cref{fig:manifold}B for a visualisation.

Next, we partition the solution manifold into distinct regions and subsequently analyse their respective representational and computational properties.
For proofs of \cref{the:least-squares-solution,the:minimum-weight-norm-solution,the:minimum-representation-norm-solution} refer to \cref{app:sec:partitioning-of-suolution-manifold}, and to \cref{fig:manifold}C-E for visualisations.
\begin{definition} \label{def:least-squares-solution}
  Any \ac{gls} that minimises the norm of the network function,
  \begin{equation}
    \argmin_{\wa, \wb}\left|\left|\wb\wa\right|\right|_F^2\ \ \text{s.t.}\ \ \wb\wa\sigxx = \sigyx
  \end{equation}
  is a \ac{lss}.
\end{definition}
\begin{theorem} \label{the:least-squares-solution}
  All \ac{lss} satisfy $\wbo\wao = \sigyx\left(\sigxx\right)^+$ and are exactly parametrised by
  \begin{equation}
    \begin{aligned}
      \wao & = \bfQ\sqrt{\bfS}\bfV^T + \bfGamma_1\bfP_\text{i} + \bfGamma_2\bfP_\text{u}\ \ \text{and} \\
      \wbo & = \bfU\sqrt{\bfS}\bfQ^+ + \bfPsi + \bfPhi + \bfGamma_3(\bfI - \wao^{}\wao^+),
    \end{aligned}
  \end{equation}
  subject to the definitions and constraints in \cref{the:general-linear-solution}, and the additional constraint that $\rank(\bfH\bfH^+\wao\bfP_\text{u}) \leq \rank((\bfI - \bfH\bfH^+)\wao\bfP_\text{u})$, and where $\bfPhi = -(\bfU\sqrt{\bfS}\bfQ^+ + \bfPsi)\wao\bfP_\text{u}\left[(\bfI - \bfH\bfH^+)\wao\bfP_\text{u}\right]^+$.
\end{theorem}
Here, $\bfPhi$ cancels interference from unobserved inputs projected into the occupied hidden space, with the rank constraint ensuring a correction exists.
\begin{definition} \label{def:minimum representation-norm-solution}
  Any \ac{gls} for which the norm of the hidden-layer representations and readout weights is minimised
  \begin{equation}
    \begin{aligned}
      \argmin_{\wa, \wb}||\wa\bfX|| & _F^2 + ||\wb||_F^2                   \\
                                    & \text{s.t.}\ \wb\wa\sigxx = \sigyx~,
    \end{aligned}
  \end{equation}
  is a \ac{mrns}.
\end{definition}
\begin{theorem} \label{the:minimum-representation-norm-solution}
  All \ac{mrns} are parametrised by
  \begin{equation}
    \wbo = \bfM\sqrt{\bfN}\bfR^T\ \ \text{and}\ \ \wao = \bfR\sqrt{\bfN}\bfO^T\bfX^+ + \bfGamma_2\bfP_\text{u},
  \end{equation}
  where $\bfR \in \mathcal{R}^{N_h\times r}$ is an arbitrary (semi\nobreakdash-)orthonormal matrix, and
  \begin{equation}
    \svd(\bfY\bfC\bfC^T) = \bfM\bfN\bfO^T.
  \end{equation}
\end{theorem}
\begin{definition} \label{def:minimum weight-norm-solution}
  Any \ac{gls} for which the sum of the norm of the weight matrices is minimised
  \begin{equation}
    \argmin_{\wa, \wb}||\wa||_F^2 + ||\wb||_F^2\ \text{s.t.}\ \wb\wa\sigxx = \sigyx
  \end{equation}
  is a \ac{mwns}.
\end{definition}
\begin{theorem} \label{the:minimum-weight-norm-solution}
  All \ac{mwns} are parametrised by
  \begin{equation}
    \wbo = \bfU\sqrt{\bfS}\bfR^T\ \ \text{and}\ \ \wao = \bfR\sqrt{\bfS}\bfV^T.
  \end{equation}
\end{theorem}
We note, that the relation between \ac{mwns} and the singular value decomposition of the least-squares solution has been previously derived under strong assumptions, namely that $\sigxx = \bfI$, $N_i = N_o$ and that $\sigyx$ has full rank \parencite[see appendix S14-S15]{saxe2019mathematical}.
Further, we note that
\begin{corollary}
  \ac{mrns} and \ac{mwns} are identical if inputs are whitened, \emph{i.e.}, $\sigxx = \bfI$.
\end{corollary}
A key difference between the four solution types lies in how they constrain the image and kernel of the weight matrices, which map between input, hidden, and output spaces.
\ac{gls} impose minimal constraints.
\ac{lss} restrict projections to and from null spaces in the input and hidden layers.
\ac{mrns} further reduce freedom in the null space, eliminate irrelevant projections, and constrain the core solution itself.
\ac{mwns}, the most restrictive class, eliminate all irrelevant and null-space projections, and enforce balance in the core by requiring equal contributions from the input and output weights.
\erin{Relate to balancedness from Clem's work.}
While this perspective clarifies how different subregions of the solution manifold constrain the structure of the weight matrices, it does not address a key question: how these solutions differ in their hidden-layer representations.

\subsection{Hidden-layer representations}
\label{sec:hidden-layer}

Understanding hidden-layer representations begins with identifying the degrees of freedom in the input weights, which govern how inputs are mapped into hidden space.
Intuitively, input weights are constrained only by the need to preserve sufficient task-relevant information for the output weights to solve the task.
In this section, we go beyond this intuition by leveraging the exact parametrisations of $\wao$ to precisely characterise the degrees of freedom in hidden-layer representations.
Proofs for \cref{cor:lss-rsm,corr:mrns-rsm,corr:mwns-rsm} are in \cref{app:sec:hidden-layer-representations}.
We begin by noting that \ac{gls} and \ac{lss} differ only in how they handle projections from unobserved input and unoccupied hidden directions and thus implement the same input-output map on the training data.
As a result, they exhibit identical degrees of freedom in their hidden-layer representations
\begin{equation}
  \bfH = \bfQ\sqrt{\bfS}\bfV^T\bfX + \bfGamma_1\bfP_\text{i}\bfX.
\end{equation}
Since $\bfQ$ can be any full-column-rank matrix and $\bfGamma_1$ is free up to a rank constraint, \ac{gls} and \ac{lss} support nearly arbitrary hidden-layer representations.
To illustrate this, we consider a semantic learning task linking items to positions within a hierarchical structure~(\cref{fig:representations}A,B).
For example, we can select a point on the solution manifold where the hidden-layer representations of the items form the shape of an elephant~(\cref{fig:representations}C).
\begin{corollary}
  \label{cor:lss-rsm}
  \ac{gls} and \ac{lss} permit any \ac{rsm} of the form
  \begin{equation}
    \begin{aligned}
      \rsm & = \bfX^T(\bfV\sqrt{\bfS}\bfQ^T\bfQ\sqrt{\bfS}\bfV^T + \bfV\sqrt{\bfS}\bfQ^T\bfGamma_1\bfP_\text{i}                       \\
           & \qquad\quad+ \bfP_\text{i}^T\bfGamma_1^T\bfQ\sqrt{\bfS}\bfV^T + \bfP_\text{i}^T\bfGamma_1^T\bfGamma_1\bfP_\text{i})\bfX.
    \end{aligned}
  \end{equation}
\end{corollary}
\begin{figure*}[t]
  \centering
  \includegraphics[width=\textwidth]{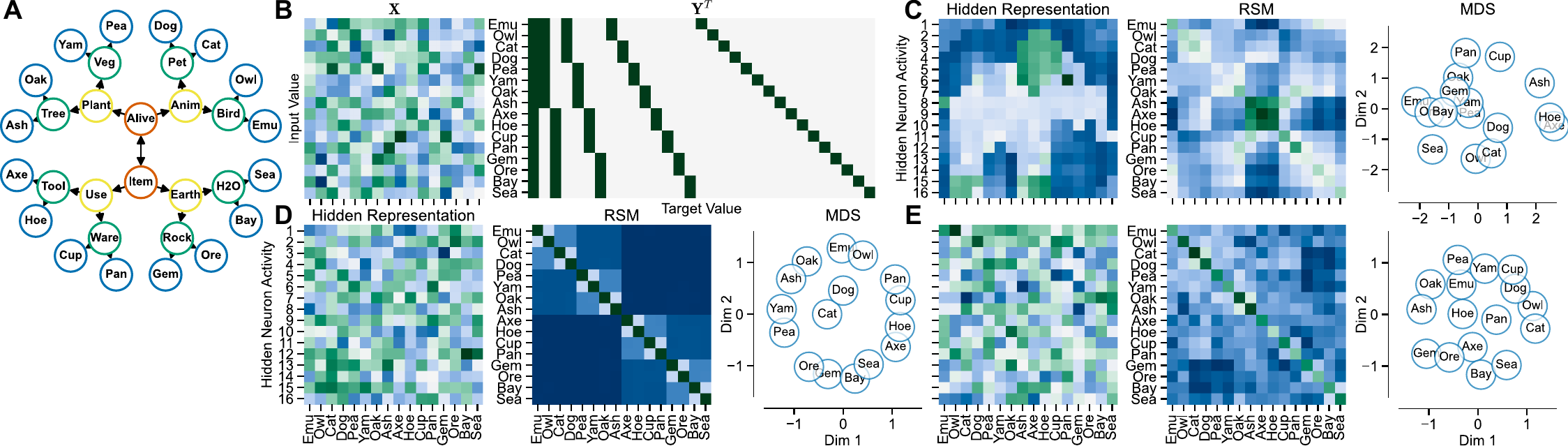}
  \caption{
    \textbf{Hidden-layer representations}.
    \fa Schematic of the semantic hierarchy task.
    \fb Inputs are encoded as random vectors (left) and corresponding target vectors encode for the position in the hierarchy (right).
    A one (zero) indicates that an item is (not) a child of a node.
    \fc Example hidden-layer representations (left), \ac{rsm} (centre) and corresponding 2D \ac{mds} plot (right) for a \ac{gls},
    \fd \ac{mrns}, and
    \fe \ac{mwns}.
  }
  \label{fig:representations}
  \vspace{-0.3cm}
\end{figure*}
In our example, this yields a highly structured \ac{rsm}, yet does not reflect the task structure, \ie the hierarchical relationships between items~(\cref{fig:representations}C).
Accordingly, we make
\begin{definition}
  Neural representations whose \ac{rsm} depends on the specific choice of input weights are \emph{task-agnostic} representations.
\end{definition}
In contrast, hidden-layer representations of \ac{mrns}
\begin{equation}
  \bfH = \bfR\sqrt{\bfN}\bfO^T\bfX^+\bfX
\end{equation}
and \ac{mwns}
\begin{equation}
  \bfH = \bfR\sqrt{\bfS}\bfV^T\bfX
\end{equation}
are unique up to an orthogonal transformation $\bfR$, which includes rotations and reflections.
Thus, in both cases, hidden-layer representations are not unique.
In the semantic hierarchy task, this results in representations that appear arbitrary and unstructured~(\cref{fig:representations}D, E).
However,
\begin{corollary}
  \label{corr:mrns-rsm}
  The \ac{rsm} of \ac{mrns} is unique and given by
  \begin{equation}
    \rsm = \bfO\bfN\bfO^T.
  \end{equation}
\end{corollary}
Since $\bfO$ and $\bfN$ are fully determined by the training data, the \ac{rsm} is invariant to the specific choice of input weights.
Similarly,
\begin{corollary}
  \label{corr:mwns-rsm}
  The \ac{rsm} of \ac{mwns} is unique and given by
  \begin{equation}
    \rsm = \bfX^T\bfV\bfS\bfV^T\bfX.
  \end{equation}
\end{corollary}
Again, the \ac{rsm} is invariant to the specific choice of input weights, as $\bfV$ and $\bfS$ are fully determined by the training data.
Accordingly, we make
\begin{definition}
  Neural representations whose \ac{rsm} if fully determined by the training data are \emph{task-specific} representations.
\end{definition}
In the semantic hierarchy task, this yields an \ac{rsm} that reflects the hierarchical structure, where representational similarity increases with proximity in the hierarchy, for \ac{mrns}; and an \ac{rsm} that reflects a combination of input statistics and target hierarchy for \ac{mwns} ~(\cref{fig:representations}D,E).

In summary, in two-layer linear networks, neural representations and the underlying function are dissociable: the same function can arise from different hidden-layer representations, and the same representations can support different functions.
For \ac{gls} and \ac{lss} this flexibility supports almost arbitrary representations and task-agnostic \acp{rsm}.
In contrast, \ac{mrns} and \ac{mwns} impose constraints which determine representations up to orthonormal transformations, and are guaranteed to have unique and task-specific \acp{rsm}.
\section{Implications for neural data analysis}
\label{sec:implications}

A fundamental challenge in understanding computations and learning in artificial and biological neural networks is linking changes in connectivity and representations to changes in function.
However, if representation and function are dissociable, two key observations follow.
First, changes in network function need not alter neural representations, as changes in one layer can be offset by compensatory changes in the next.
Second, changes in network function can proceed without altering representations, as it can, in principle, occur entirely in downstream layers.
We now examine the implications of these findings for representational comparisons, representational drift, and the synaptic stability-plasticity trade-off through a series of illustrative simulations.
While exact outcomes depend on task specifics and hyperparameter choices, these are not our focus here; a systematic analytical and numerical exploration is left for future work.

\subsection{Linear predictivity}
\begin{figure*}[!t]
  \centering
  \includegraphics[width=\textwidth]{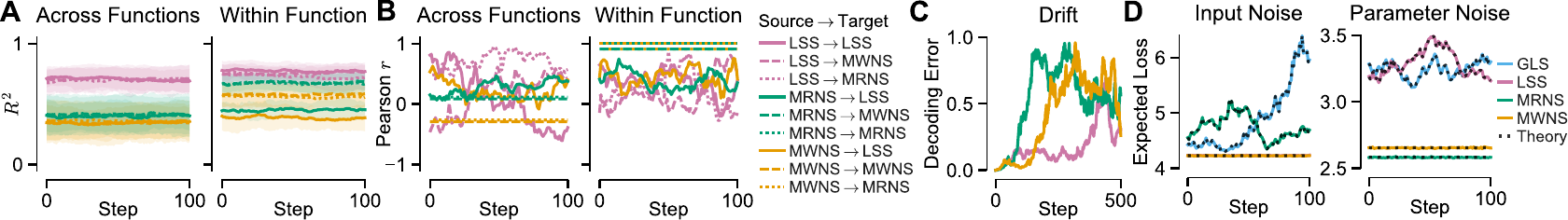}
  \caption{
    \textbf{Implications for neural data analysis.}
    All panels show results during random walks on the solution manifolds of \ac{lss}, \ac{mwns}, and \ac{mrns}.
    \fa Mean and standard deviation of $R^2$ scores for linear predictivity across $n=10$ random walks.
    All source-target combinations are shown for across-function (left) and within-function (right) comparisons.
    \fb Example trajectories of RSA correlation scores, shown for across-function (left) and within-function (right) comparisons.
    \fc MSE of a linear decoder trained on the hidden-layer representation at the initial time step.
    \fd Mean and expected MSE under input noise (left) and parameter noise (right).
  }
  \label{fig:decoding}
  \vspace{-0.3cm}
\end{figure*}
A common method for comparing neural representations is to assess how well activation patterns from a source model or recording can predict those of a target via linear regression~\parencite[\eg][]{yamins2014performanceoptimized,yamins2016using}.
High linear predictivity is often interpreted as evidence that two systems process information similarly.
However, since function and hidden-layer representations are dissociable, strong linear predictivity does not necessarily imply functional alignment.
We illustrate this by comparing hidden-layer representations from independent random walks on the solution manifold of \ac{lss}, \ac{mwns}, and \ac{mrns}~(\cref{fig:decoding}A).
In each case, we compare either representations from two different network functions (across function) or from the same network function (within function).
While $R^2$ scores are on average slightly higher for within-function comparisons, the main determinant of predictivity is whether the representations are task-agnostic or task-specific.
Specifically, in across-function comparisons, $R^2$ scores are highest when the source representation comes from a \ac{lss} (which shares representational degrees of freedom with \ac{gls}), followed by \ac{mrns}, and lowest for \ac{mwns}.
Within-function comparisons yield high $R^2$ when the source is task-agnostic or when both source and target are of the same type; in contrast, predicting a task-agnostic representation from a task-specific one lead to the lowest $R^2$ scores.
These patterns arise because task-agnostic solutions process both relevant and irrelevant input directions, producing higher-rank representations that cannot be linearly predicted from the lower-rank task-specific ones.
These results indicate that linear predictivity is predominantly driven by solution type rather than functional alignment, and may yield misleading conclusions if underlying representational constraints are not explicitly taken into account.

\subsection{Representational similarity analysis}
\Ac{rsa} compares neural activation patterns by evaluation the similarity of \acp{rsm} across conditions, stimuli, models, or participants~\parencite{kriegeskorte2008representational, haxby2014decoding}.
As with linear predictivity, our analytical results show that the interpretability of \ac{rsa} in terms of functional alignment critically depends on the solution type of the comparanda.
We illustrate this in~\cref{fig:decoding}B using example trajectories from the previous section.
Since task-agnostic solutions (\ie \ac{gls} and \ac{lss}) exhibit highly flexible \acp{rsm}, correlation coefficients $r$ involving such representations fluctuate unpredictably throughout the random walk, both within and across functions.
By contrast, comparing within and across task-specific solutions (\ie \ac{mwns} and \ac{mrns}) results in static and consistent $r$, as they exhibit unique \acp{rsm}.
However, because \ac{mwns} and \ac{mrns} induce different unique \acp{rsm}, comparisons across types yield imperfect similarity even within the same function.
In summary, \ac{rsa} reliably reflects functional similarity only when representational constraints enforce unique \acp{rsm}, underscoring the importance of accounting for solution type in representational comparisons.

\subsection{Drifting neural representations}
Intuitively, one might expect that if a stimulus elicits stable perception and behaviour, the associated neural representations should likewise remain stable~\parencite{rule2019causes, driscoll2022representational}.
However, this assumptions is challenged by converging evidence of representation drift across species, brain regions and modalities~\parencite[\eg][]{ziv2013long, driscoll2017dynamic, schoonover2021representational, marks2021stimulus, deitch2021representational, alisha2023representational}.
Our analysis shows that the existence of a solution manifold allows hidden-layer representations to vary without changing the implemented function.
This dissociation implies that stable perception and behaviour do not require stable representations.
Indeed, an optimal linear decoder trained on an initial representation rapidly degrades in performance during a random walk on the solution manifold~(\cref{fig:decoding}C).
Thus, representational drift need not signal functional change, but may instead reflect a reparametrisation within a functionally equivalent subspace.

\subsection{Synaptic stability and plasticity}
The so-called stability-plasticity dilemma posits that neural systems must remain plastic enough to acquire new knowledge while remaining stable enough to retain previously learned information~\parencite{grossberg1987competitive, abraham2005memory}.
This view, grounded in single-neuron and synapse-level intuitions, has motivated continual learning algorithms that explicitly regulate synaptic changes to preserve past knowledge ~\parencite[\eg][]{kirkpatrick2017overcoming, zenke2017continual, aljundi2018memory}.
However, our analysis shows that this dilemma need not apply at the network level, because, independent of the solution type, many distinct configurations of synapses and representations implement the same function~(see \eg \cref{fig:intro}).
This raises important methodological concerns: observing or inducing isolated synaptic or representational changes may not suffice to infer function, learning, or the underlying learning mechanisms in artificial or biological neural networks.
\section{Advantages of task-specific representations}
\label{sec:advantages}

A natural question arises from the observation that function and representation are dissociable: why do biological and artificial systems often converge to non-arbitrary representations that obey the structure of the task?
One possible explanation is that task-specific representations confer computational advantages, creating selective pressure in both biological and artificial systems.
Consequently, such representations may emerge as preferred functional implementations, leading to representational alignment across systems.

\subsection{Secondary error}
\label{sec:secondary-error}
One hypothesised advantage of task-specific representations is improved performance on secondary datasets, such as in- or out-of-distribution generalisation.
To test this, we identify solutions on the primary-task solution manifold that minimise the error on an unseen secondary dataset $\tilde{\mathcal{D}}=\{(\tilde{\bfx}_n, \tilde{\bfy_n})\}_{n=1}^Q$.
In~\cref{app:sec:secondary-error} we derive,
\begin{theorem}
  \label{the:secondary-error}
  The secondary error is minimised by all solutions of the form
  \begin{equation}
    \wbo\wao = \sigyx\sigxx^+ + \tilde{\bfZ}\bfP_\text{u},
  \end{equation}
  where
  \begin{equation}
    \tilde{\bfZ} = \left(\bfYt - \sigyx\sigxx^+\bfXt\right)\left(\bfP_\text{u}\bfXt\right)^+ + \tilde{\bfGamma}\tilde{\bfP}_\text{u},
  \end{equation}
  with $\tilde{\bfGamma} \in \mathcal{R}^{N_o \times N_i}$ arbitrary and $\tilde{\bfP}_\text{u} = \bfI - \bfP_\text{u}\bfXt(\bfP_\text{u}\bfXt)^+$.
\end{theorem}
Since the solution is a \ac{lss} with a perturbation in the unobserved null directions of the primary-task inputs, minimising secondary error permits task-agnostic solutions.
Thus, observing $\bfH$ gives no information about secondary task performance and secondary error cannot explain the emergence of task-specific representations in two-layer linear networks.

\subsection{Sensitivity to noise}
\label{sec:robustness}

\erin{Tease out relationship between noise-robustness and flat minima, , see \url{https://arxiv.org/abs/1802.05296}}

Neural systems are subject to a multitude of internal and external sources of noise, which range from variability in incoming sensory signals to fluctuations in synaptic efficacy and spontaneous neural activity~\parencite{faisal2008noise}.
Therefore, solutions that exhibit robustness to such noise are advantageous, as they enable the neural circuitry to maintain reliable function~\parencite{johnston2020nonlinear}.
The following theorems are derived in~\cref{app:sec:noise-proofs}.
\begin{theorem} \label{the:input-noise}
  The expected loss under additive, \ac{iid}, zero-centred input noise $\bfxi_{\bfx_n}$ with variance $\sigma_x^2$ is
  \begin{equation}
    \begin{aligned}
       & \bigg\langle\frac{1}{2P}\sum_{n=1}^P||\wbo\wao\left(\bfx_n + \bfxi_{\bfx_n}\right) - \bfy_n||_2^2\bigg\rangle \\
       & \ = \frac{\sigma_\bfx^2}{2}||\wbo\wao||_F^2 + c,
    \end{aligned}
  \end{equation}
  where $c$ is a noise-independent constant that only depends on the training data.
\end{theorem}
\begin{corollary}
  The expected loss under input noise is exclusively minimised by \ac{lss}, which include \ac{mwns} as a subspace (see~\cref{fig:decoding}D).
\end{corollary}
Thus, while robustness to input noise selects for solutions with minimal functional norm, \ac{lss} employ task-agnostic representations, so robustness alone does not ensure representational alignment.
\begin{theorem} \label{the:parameter-noise}
  The expected loss under additive, \ac{iid}, zero-centred noise $\Xi_1$ and $\Xi_2$ in the parameters, with variances~$\sigma_1^2 \propto 1 / ||\bfX||_F^2$ and~$\sigma_2^2 \propto 1 / N_o$ is
  \begin{equation}
    \begin{aligned}
       & \la \frac{1}{2P}\sum_{n=1}^{P}||\left(\wbo + \bfXi_2\right)\left(\wao + \bfXi_1\right)\bfx_n - \bfy_n||^2_2 \ra \\
       & \ = \frac{1}{2P}\Big(||\wao\bfX||_F^2 + ||\wbo||_F^2 + c\Big).
    \end{aligned}
  \end{equation}
  where $c$ is again a noise-independent constant.
\end{theorem}
\begin{corollary}
  The expected loss under parameter noise is exclusively minimised by \ac{mrns} (see~\cref{fig:decoding}D).
\end{corollary}
\begin{figure*}[!t]
  \centering
  \includegraphics[width=\textwidth]{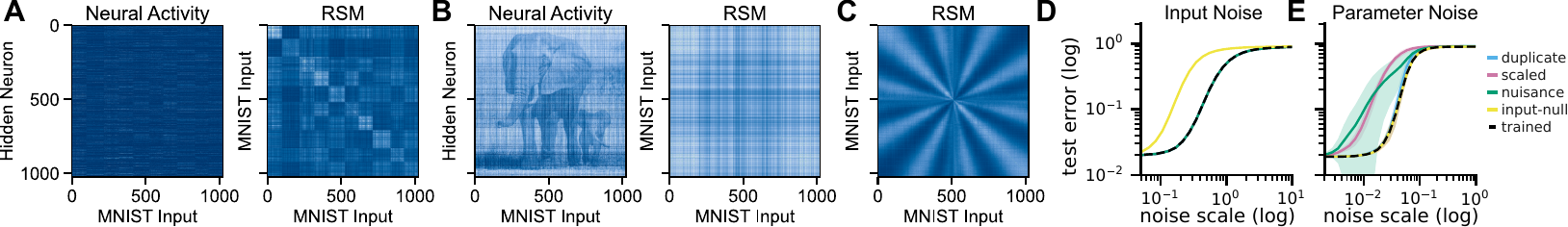}
  \caption{
    \textbf{Function and representation are dissociable in non-linear networks.}
    \fa Hidden-layer activations for $1024$ MNIST inputs, grouped by class (left) and the corresponding task-specific \ac{rsm} (right) after training a ReLU network from small initial weights.
    \fb Same as \fa, but for a network reparametrised via augmented Lagrangian optimisation to reshape hidden-layer representation while preserving training set classifications.
    \fc \ac{rsm} of the network from \fa after reparametrisation using exact invariant transformations (\cref{sec:nonlinear-symmetries}).
    \fd Test error under input noise for all exact invariant transformations.
    Networks with \texttt{input-null} expansion are sensitive.
    \fe As in \fd but for parameter noise; networks with \texttt{scaled}, \texttt{nuisance}, and \texttt{duplicate} expansions are sensitive to varying degrees.
  }
  \label{fig:mnist}
  \vspace{-0.3cm}
\end{figure*}
We have scaled noise variances to simplify the analytical expression; without this scaling, the results hold up to multiplicative constants.
Robustness to parameter noise thus selects for solutions with task-specific representations, ensuring representational alignment.
\begin{theorem} \label{the:parameter-and-input-noise}
  Under the assumption that the input data is whitened, \ie $\sigxx = \bfI$, the expected loss under additive, \ac{iid}, zero-centred parameter noise $\bfXi_1$ and $\bfXi_2$ with variance $\sigma^2_1 \propto 1 / N_i$ and $\sigma^2_2 \propto 1 / N_o$ and input noise $\bfxi_{\bfx_n}$ with variance $\sigma_x^2$ is
  \begin{align}
     & \bigg\langle\frac{1}{2P}\sum_{n=1}^P||\left(\wbo + \bfXi_2\right)\left(\wao + \bfXi_1\right)\left(\bfx_n + \bfxi_{\bfx_n}\right) - \bfy_n||_2^2\bigg\rangle \nonumber \\
     & = \frac{\sigma^2_{\bfx}}{2}\big(||\wbo\wao||_F^2 + ||\wbo||_F^2 + ||\wao||_F^2\big)                                                                                   \\
     & \qquad\qquad + \frac{1}{2P}\big(||\wbo||_F^2 + ||\wao\bfX||_F^2\big) + c,\nonumber
  \end{align}
  with noise-independent constant $c$.
\end{theorem}
\begin{corollary}
  Under the stated assumptions and constraints, the expected loss under input and parameter noise is minimised exclusively by \ac{mrns} and \ac{mwns}.
\end{corollary}
We note, that if the input data is not whitened, interaction terms render the optimal subspace depends on input statistics, complicating its explicit analytical characterisation.

In summary, neither minimising secondary error nor input noise sensitivity promotes task-specific representations.
In contrast, robustness to parameter noise selectively favours solutions with task-specific structure, thereby supporting representational alignment.
This suggests that shared representations across biological and artificial systems may arise from implicit or explicit optimisation for parameter robustness, such as regularisation strategies that favour low-norm solutions.
However, this reflects an inductive bias rather than a general principle and without explicit justification, functional and representational alignment cannot be assumed to coincide.
\section{Nonlinear networks}
\label{sec:nonlinear}
The results presented thus far apply to two-layer linear networks.
We now extend our study to \emph{nonlinear} networks (networks with nonlinear activation functions), and show that they exhibit analogous degrees of freedom in representation and associated computational trade-offs as their linear counterparts.
Here, we face a challenge: A full characterisation of the solution manifold of general nonlinear network remains analytically intractable, even for two-layer networks~\parencite{misiakiewicz2024six}.
However, substantial progress has been made in deriving function-preserving transformations that allow reparametrisation of nonlinear networks while leaving their input-output map invariant~\parencite{simsek2021geometry,martinelli2023expandandcluster}.
Here, we exploit these function-preserving transformations to construct functionally equivalent reparametrisations of nonlinear networks, which allows us to probe computational differences between networks with minimal and expanded representations, in analogy to the task-specific and task-agnostic representations of \cref{sec:manifold}.

\subsection{Functional invariances in deep ReLU networks}
\label{sec:nonlinear-symmetries}
Feedforward networks with any activation function are output-invariant to permuting neurons within a layer, as permuting the rows of one weight matrix and the corresponding columns of the next leaves the network function unchanged~\parencite[\emph{permutation invariance};][]{sussmann1992uniqueness}.
Feedforward networks with \ac{relu} activation are further invariant to rescaling a neuron's incoming weights by a factor $\alpha > 0$ and its outgoing weights by $1/\alpha$, due to the non-negative homogeneity of \ac{relu}~\parencite[\emph{scale invariance};][]{neyshabur2015pathsgd}.
\textcite{simsek2021geometry,martinelli2023expandandcluster} identify additional invariances that fully characterise the manifold of global minima in teacher-student settings, where one network is trained to replicate the function of another.
Although these invariances may not capture all functionally equivalent parametrisations outside of the teacher-student setting, they provide a means to construct network reparametrisations that exactly preserve network function.
We realise two of their invariances by inserting hidden-layer neurons with arbitrary incoming weights and zero outgoing weights (\emph{nuisance-neuron invariance}) and by duplicating hidden-layer neurons while halving the outgoing weights of both the original and the duplicated neuron (\emph{duplication invariance}).
Lastly, one can add any perturbation to the input-layer weights that lies in the unobserved nullspace of the input data while preserving the network's function (\emph{input-nullspace invariance}).

\subsection{Manipulating representations of ReLU networks}
\erin{Difficulties with ReLU equivalent to result from \url{https://openreview.net/pdf?id=7CUluLpLxV\#subsection.9.2}}

We train a two-layer \ac{relu} network with $1024$ hidden neurons on the MNIST dataset~\parencite{lecun1998gradient} from small norm random weights, resulting in task-specific representations~\parencite[\cref{fig:mnist}A, ``rich'' learning from][]{jacot2018neural,chizat2019lazy,woodworth2020kernel}.
Starting from this trained model, we use augmented Lagrangian optimisation to modify the hidden-layer activations of $1024$ training inputs such that their representations collectively resemble an image of two elephants, while enforcing that the network's predicted class labels remain unchanged across the entire training set~(\cref{fig:mnist}B).
This transformation illustrates the representational freedom among nonlinear \ac{relu} networks that make the same class predictions.
In addition, even when limited to manipulations that exactly preserve a network's input-output map (a stricter condition than preserving classifications) one can induce a task-agnostic \ac{rsm}~(\cref{fig:mnist}C).

\subsection{Computational advantages in ReLU networks}
To complement the analytical results in~\cref{sec:robustness} on the robustness of task-specific representations in linear networks, we empirically evaluate secondary (test) error and robustness to input and parameter noise across different nonlinear solution types.
We train two-layer networks with 1024 hidden dimensions and \ac{relu} activation on the training set of the MNIST digit classification task~\parencite{lecun2010mnist} from 8 random initialisations.
These models trained from small initial weights serve our \emph{task-trained} (minimal) solutions.
We next apply four function-preserving transforms defined by the four invariances of \cref{sec:nonlinear-symmetries} to these task-trained (minimal) networks to construct expanded (non-minimal) parametrisations that exactly preserve the network function.
These minimal and non-minimal parametrisations serve as our solution types in the nonlinear setting.

In~\cref{fig:mnist}D, we observe the effect of \textbf{input noise} on the task-trained (minimal) and expanded (non-minimal) networks.
In accordance with the linear result, only transformations in the unobserved input space have a deleterious effect (\texttt{input-null}).
In~\cref{fig:mnist}E, we observe the effect of \textbf{parameter noise} on the initial and transformed networks.
Non-minimal models (\texttt{scaled}, \texttt{nuisance}, \texttt{duplicate}) degrade in test error more quickly than minimal ones, similarly to what is derived in \cref{sec:robustness}, though duplicated expansions (\texttt{duplicate}) are more robust due to noise averaging.
In contrast, input-nullspace perturbations (\texttt{input-null}) have no effect because transformations are in unobserved input directions, and input noise is absent.
Lastly, at near-zero noise levels, we observe that no manipulations inflate the secondary (test) error, consistent with the result of \cref{sec:secondary-error} that generalisation performance does not constrain network representations to be minimal.
\section{Related work}
\label{sec:related-work}

\paragraph{The solution manifold of artificial networks.}

The solution manifold of two-layer neural networks was first described by \textcite{baldi1989neural}.
Subsequent work showed under some assumptions that all minima in deep linear networks are global and equivalent to those in linear regression \parencite{laurent2018deep}.
The dissociation between general linear solutions and minimum-norm solutions has been previously studied under a set of strong assumptions \parencite{saxe2014exact}.
Sensitivity to noise for task-specific and task-agnostic rich and lazy solutions has been previously studied numerically in nonlinear neural networks \parencite{flesch2022orthogonal} and generalisation and transfer performance of deep linear networks have been previously investigated \parencite{lampinen2019analytic, advani2020highdimensional,tahir2024features,ingrosso2024statistical} using a teacher-student paradigm \parencite{gardner1989three,riegler1995online, saad1995online}.
The relation between representational drift and drift on the solution manifold that results from stochasticity during gradient descent \parencite{chaudhari2018stochastic} has been studied on a subpart of the solution manifold in linear networks \parencite{pashakhanloo2023stochastic} and in nonlinear networks, again, relying on the teacher-student setting \parencite{avidan2024connecting, li2024representations}.
\erin{state a tiny bit how we differ}

\paragraph{Comparing the solutions of artificial and biological networks.}

Neuroscientists have identified parallels between artificial and biological neural computation~\parencite{richards2019deep, saxe2021if, doerig2023neuroconnectionist} from hierarchical feature extraction in visual processing \parencite{dicarlo2012does, eickenberg2017seeing,lindsay2021convolutional} to analogous population dynamics during decision-making tasks \parencite{mante2013context, chaisangmongkon2017computing}.
The field has developed various methods to quantify this shared representational structure, including representational similarity analysis \parencite{kriegeskorte2008representational}, linear predictivity~\parencite[\eg][]{yamins2014performanceoptimized,yamins2016using}, and metric-based methods \parencite{williams2021generalized}; see~\textcite{klabunde2023similarity} for an overview of methods.
However, recent work has identified significant methodological challenges in comparing representations between artificial neural networks, including confounding effects from stimulus correlations \parencite{cai2019representational,hermann2020what,dujmovic2023obstacles}, metric-dependent results \parencite{ding2021grounding,bo2025evaluating}, and difficulties in matching representations even between identical networks trained with different random initialisations~\parencite{han2023system}; these negative results suggest further problems when comparing artificial and biological neural networks, where little is known in advance of the computation the biological networks performs.
These challenges may apply to existing neural predictivity benchmarks such as Brain-score~\parencite{schrimpf2018brainscore,schrimpf2020integrative} and the Natural Scenes Dataset \parencite[NSD;~][]{allen2022massive}.

\section{Discussion}

In this work, we give a complete analytical characterisation of the global minima manifold for deep linear networks, and demonstrate that different subregions of this manifold afford different interpretability and computational properties due to their representational structure.
We conclude that the use of deep, overparametrised networks poses fundamental challenges for representational analysis, interpretation, and comparison, as the impact of variability in the parametrisation of functionally equivalent representations on these use cases is significant.

Our analysis does not assume a specific learning algorithm and ignores the question of how a specific solution could be attained in practice.
However, the computational advantages of task-specific representations detailed in \cref{sec:advantages} are compatible with the view that gradient descent, in particular in overparametrised models, is subject to \emph{implicit regularisation} that prefers solutions with certain optimality properties for both linear and nonlinear networks~\parencite{neyshabur2015search,  zhang2017understanding, neyshabur2017implicit, du2019gradient, arora2018optimization, chizat2020implicit, yun2021unifying, vardi2021implicit}.

\section*{Acknowledgements}
L.B.
was supported by the Woodward Scholarship awarded by Wadham College, Oxford and the Medical Research Council (MR/N013468/1).
E.G. and A.S.
were supported by a Schmidt Science Polymath Award to A.S., and the Sainsbury Wellcome Centre Core Grant from Wellcome (219627/Z/19/Z) and the Gatsby Charitable Foundation (GAT3850).

\section*{Impact statement}
This paper presents work whose goal is to advance the field of machine learning.
There are many potential societal consequences of our work, none of which we feel must be specifically highlighted here.

\printbibliography

\newpage
\appendix
\onecolumn
\section{Simulation details} \label{app:sec:simulation-details}

A code repository reproducing all figures is available on GitHub at \href{https://github.com/lukas-braun/dissociating-similarity}{\texttt{lukas-braun/dissociating-similarity}}.

\subsection{Random walk}
Given a two-layer linear network with weight matrices $\wa$ and $\wb$, we implemented random walks on the \ac{gls} as follows.
First, we randomly initialised weight matrices using a random normal initialisation with zero mean and standard deviation $1 / \sqrt{N_i}$ and $1 / \sqrt{N_h}$ respectively.
Then, for each step, we first diffused both weight matrices by
\begin{equation} \label{app:eq:diffusion}
  \bfW(t+1) = (1 - \alpha)\bfW(t) + \alpha\mathbf{\xi}\quad
\end{equation}
with
\begin{equation}
  \mathbf{\xi} \sim \mathcal{N}(\mu=0, \sigma^2=25),\text{ and } \alpha= 0.00625
\end{equation}
and then performed gradient descent on both $\wa$ and $\wb$ with learning rate $\eta=0.25$ to return back to the solution manifold.
Similarly, for the random walk on the \ac{lss}, we sampled initial $\wa$ and $\wb$ randomly, diffused them, performed gradient descent on both $\wa$ and $\wb$, and subsequently enforced the two rank constraints to return to the solution manifold.
The random walk on the \ac{mrns} was implemented by first computing
\begin{equation}
  \svd(\bfY\bfX^T\bfX^{+T}) = \bfM\bfN\bfO^T,
\end{equation}
sampling a random
\begin{equation}
  \bfGamma \sim \mathcal{N}(\mu=0, \sigma^2=1/N_i),\quad\text{and}\quad\tilde{\bfR} \sim \mathcal{N}(\mu=0, \sigma^2=1).
\end{equation}
Then, in every step, we diffused $\bfGamma$ and $\bfR$ according to~\cref{app:eq:diffusion}, and then calculated
\begin{equation}
  \bfR = \bfR_1^{}\bfR_2^T\quad\text{with}\quad\svd(\tilde{\bfR}) = \bfR_1^{}\bfD\bfR_2^T.
\end{equation}
We then set
\begin{equation}
  \wa \rightarrow \bfR\sqrt{\bfN}\bfO^T \bfX^+ + \bfGamma\quad\text{and}\quad\wb \rightarrow \bfM\sqrt{\bfN}\bfR^T.
\end{equation}
Similarly, for \ac{mwns} we first computed the \ac{csvd} of the \ac{lss}
\begin{equation}
  \svd\left(\sigyx(\sigxx)^+\right) = \bfU\bfS\bfV^T
\end{equation}
and sampled a random matrix $\tilde{\bfR} \sim \mathcal{N}(\mu=0, \sigma^2=1)$.
Then, for each step, we defused $\tilde{\bfR}$ according to~\cref{app:eq:diffusion} and then calculated $\bfR$ as for \ac{mwns}.
Network weights were then set to
\begin{equation}
  \wa \rightarrow \bfR\sqrt{\bfS}\bfV^T
\end{equation}
and
\begin{equation}
  \wb \rightarrow \bfU\sqrt{\bfS}\bfR^T.
\end{equation}

\subsection{Tasks}
\textbf{Random regression task}
Random regression tasks are made up of a randomly sampled input
\begin{equation}
  \bfX \sim \mathcal{N}(\mu=0, \sigma^2=1 / N_i) \in \mathcal{R}^{N_i \times P}
\end{equation}
and target matrix
\begin{equation}
  \bfY \sim \mathcal{N}(\mu=0, \sigma^2=1 / N_o) \in \mathcal{R}^{N_o \times P}.
\end{equation}

\textbf{Semantic hierarchy}
Each of the $P=16$ items in the semantic hierarchy task was encoded as a random normal distributed vector
\begin{equation}
  \bfX \sim \mathcal{N}(\mu=0, \sigma^2=1 / N_i) \in \mathcal{R}^{N_i \times P}.
\end{equation}
Corresponding target vectors were then generated according to the position of the item in the hierarchy.
If an item is a child of a particular node it is encoded as a $1$ and $0$ otherwise.
For example, a pea is alive, not an object, not an animal, a plant, not handy, not earth, not a bird, not furry, but a veg, not a tree ... resulting in the target vector [1, 0, 0, 1, 0, 0, 0, 0, 1, 0 ...].

\subsection{Figure 1}
Random walk on \ac{mrns} manifold of random regression task with $N_i=4$, $N_h=7$, $N_o=3$, and $P=3$ for $500$ steps with diffusion parameter $\alpha=0.005$.

\subsection{Figure 3}
The neural network has $N_i=16$, $N_h=16$, $N_o=31$.
The point on the \ac{lss} solution manifold which depicts an elephant $bfE \in \mathcal{R}^{N_h \times P}$, was found by first initialising the weight matrices as
\begin{equation}
  \wa = \bfE\bfX^T(\bfX\bfX^T + 5 \times 10^{-3} \bfI)^+\quad\text{and}\quad \wb = \bfY\bfX^T(\wa\bfX(\wa\bfX)^T + 1 \times 10^{-2}\bfI)^+,
\end{equation}
and subsequently applying gradient descent on both weight matrices to find a point on the solution manifold.
Randomly sampled points on the solution manifold of \ac{mrns} and \ac{mwns} were generated according to the definitions of $\wao$ in \cref{the:minimum-representation-norm-solution}, and \cref{the:minimum-weight-norm-solution}.

\subsection{Figure 4}
Random walk on solution manifold for \ac{gls}, \ac{lss}, \ac{mrns}, and \ac{mwns} on a random regression task with $N_i=6$, $N_h=16$, $N_o=2$, and $P=4$ for $100$ steps.

\textbf{(A)}
We fit a linear model to predict the hidden-layer activations of one model by the hidden-layer activations of another for each time step of the random walk, using ridge regression and subsequently calculated the $R^2$ score of that fit.

\textbf{(B)} We performed representational similarity analysis on pairwise random walk trajectories of (B) by calculating the euclidean distance matrix of their respective hidden-layer representations
\begin{equation}
  \rsm_{ij} = ||\bfh_i - \bfh_j||_2
\end{equation}
and subsequently compared their off-diagonal values using Pearson correlation coefficients.

\textbf{(C)} To analyse drift, we fit a linear decoder on the hidden-layer representation of a model at the beginning of the random walk
\begin{equation}
  \tilde{\bfW} = \sigyx \wa^T (\wa\bfX\bfX^T\wa^T)^+
\end{equation}
and subsequently calculated the mean squared error of that classifier, given the hidden-layer representation at each time step.

\textbf{(D)}
We numerically calculated the average loss across $n=500000$ randomly sampled input noise vectors with $\sigma^2_\bfx=1$ at each time step.
Accordingly, we numerically calculated the average loss across $n=500000$ randomly sampled parameter noise matrices with $\sigma_1^2 = 1 / ||\bfX||^2_F$ and $\sigma_2^2 = 1 / N_o$.

\subsection{Simulating nonlinear networks}
As a precursor to all numerical experiments with nonlinear networks in \cref{sec:nonlinear}, we train feed-forward neural networks on a non-synthetic dataset to produce a \emph{task-specific parametrisation}.
All neural networks in this section are trained with back-propagation and (mini)-batch gradient descent from small initial weights, a regime that induces task-specific feature learning in nonlinear networks similarly to linear networks~\parencite{chizat2019lazy}.

\subsection{Expanding nonlinear networks}
\label{sec:expand-nonlinear}
The first expansion (\texttt{scaled} in figures) rescales the input weights to each neuron by a constant factor $\alpha$ and the output weights the neuron by $1/\alpha$, which preserves the output of each neuron due to the scale-invariance of \acp{relu}.
However, the magnitude of the representations that this model employs increases.
The second expansion (\texttt{nuisance} in figures) adds nuisance neurons with random incoming weights and zero outgoing weights~\parencite[``zero-type'' neurons per~][]{martinelli2023expandandcluster}; similarly, these neurons do not affect the output the model, but introduce noise in hidden-layer representations.
Lastly, we include a parameter-expanded baseline (\texttt{duplicated}) that duplicates each neuron and correspondingly rescales outgoing weights from the duplicated neuron and its copy by a factor of one half.

\subsection{Figure 5}
The neural networks start with $N_i=784$, $N_h=1024$, $N_o=10$.
Manipulations adding parameters in panel (E) add twice the number of neurons.
Models are trained for 30 epochs on the full MNIST training set of 50,000 images with exponential learning rate decay.

\section{Notation and preliminaries}
Throughout this paper we adhere to the following notation: Scalars are denoted by letters (\eg $a$, $\tau$, $L$), matrices by bold uppercase letters (\eg $\bfX$, $\bfGamma$), column vectors are denoted by bold lowercase letters (\eg $\bfx$, $\bflambda$), and row vectors by the transpose of a column vector (\eg $\bfx^T$, $\bflambda^T$).
The vector dot product and vector outer product are denoted by $\bfa^T\bfb$ and $\bfa\bfb^T$ respectively.
The zero matrix and identity matrix of size $n$ are denoted by $\bfZero$ and $\bfI_n$ respectively.
$\bfA^{-1}$ and $\bfA^{+}$ denote the inverse and Moore–Penrose pseudo-inverse of a matrix.
The $\ell_2$-norm of a vector is expressed by $||\bfx||_2$ and the Frobenius norm of a matrix is expressed by $||\bfA||_F$.
Finally, the expected value and trace operator are denoted by $\langle\cdot\rangle$ and $\tr(\cdot)$.

\subsection{Compact singular value decomposition} \label{app:sec:csvd}
We extensively use the \ac{csvd} to analyse matrix structures.
For any matrix $\bfA \in \mathcal{R}^{m \times n}$ with rank $r$, the \ac{csvd} decomposes it into a product of three:
\begin{equation}
  \svd(\bfA) = \bfU\bfS\bfV^T,
\end{equation}
where $\bfU \in \mathcal{R}^{m \times r}$ and $\bfV \in \mathcal{R}^{n \times r}$ are (semi-)orthonormal matrices containing the left and right singular vectors, and $\bfS \in \mathcal{R}^{r \times r}$ is a diagonal matrix with corresponding non-zero singular values in descending order.
In contrast to the full singular value decomposition, the singular vectors of the \ac{csvd} are not generally square orthogonal matrices but may be semi-orthogonal dependent on the relationship of $m$, $n$, and $r$ (\cref{app:tab:orthonormality-of-singular-vectors}).
It further follows, that $\bfS^{-1}$ is well defined.
Using that $(\bfB\bfC)^+ = \bfC^+\bfB^+$ if $\bfB$ has orthonormal columns or $\bfC$ has orthonormal rows \parencite{greville1966note} it further follows that
\begin{equation}
  \begin{aligned}
    \bfA^+ & = \left(\bfU\bfS\bfV^T\right)^+ \\
           & = \bfV\left(\bfU\bfS\right)^+   \\
           & = \bfV\bfS^+\bfU^T              \\
           & = \bfV\bfS^{-1}\bfU^T
  \end{aligned}
\end{equation}
is the Moore-Penrose inverse of any matrix $\bfA$.

\begin{table*}[h]
  \caption{Orthonormality of singular vectors of the \ac{csvd}}
  \label{app:tab:orthonormality-of-singular-vectors}
  \begin{center}
    \begin{small}
      \begin{sc}
        \begin{tabular}{ |c|c|c|c|c|c|c| }
          \hline
                       & \multicolumn{2}{c|}{$m = n$} & \multicolumn{2}{c|}{$m < n$} & \multicolumn{2}{c|}{$m > n$}                                              \\
          \hline
                       & $r = m$                      & $r < m$                      & $r = m$                      & $r < m$       & $r = n$    & $r < n$       \\
          \hline
          $\bfU^T\bfU$ & $= \bfI_m$                   & $= \bfI_r$                   & $= \bfI_m$                   & $= \bfI_r$    & $= \bfI_n$ & $= \bfI_r$    \\
          \hline
          $\bfU\bfU^T$ & $= \bfI_m$                   & $\neq \bfI_m$                & $= \bfI_m$                   & $\neq \bfI_m$ & $= \bfI_m$ & $\neq \bfI_m$ \\
          \hline
          $\bfV^T\bfV$ & $= \bfI_m$                   & $= \bfI_r$                   & $= \bfI_m$                   & $= \bfI_r$    & $= \bfI_n$ & $= \bfI_r$    \\
          \hline
          $\bfV\bfV^T$ & $= \bfI_m$                   & $\neq \bfI_m$                & $\neq \bfI_n$                & $\neq \bfI_n$ & $= \bfI_n$ & $\neq \bfI_n$ \\
          \hline
        \end{tabular}
      \end{sc}
    \end{small}
  \end{center}
\end{table*}

\subsection{Linearity of expected value and trace}
Both, the expected value and trace operator are linear and therefore additive
\begin{equation}
  \langle\bfA + \bfB\rangle = \langle\bfA\rangle + \langle\bfB\rangle,
\end{equation}
\begin{equation}
  \tr\left(\bfA + \bfB\right) = \tr(\bfA) + \tr(\bfB),
\end{equation}
and homogeneous
\begin{equation}
  \langle a\bfA \rangle = a \langle\bfA\rangle,
\end{equation}
\begin{equation}
  \tr(a\bfA) = a \tr(\bfA),
\end{equation}
from which it further follows that
\begin{equation}
  \langle\tr(\bfA)\rangle = \tr(\langle\bfA\rangle).
\end{equation}

\subsection{Expected values of random vectors and matrices}
Let $\bfxi_1 \in \mathcal{R}^{n}$ be a vector whose entries are \ac{iid} and drawn from a zero-centred random distribution with variance $\sigma^2_1$.
Then,
\begin{equation}
  \langle\bfxi_1\rangle = \bfZero \in \mathcal{R}^{n}
\end{equation}
and the expected value of the inner and outer product is
\begin{equation} \label{app:eq:random-vector-inner-and-outer-product}
  \langle\bfxi_1^T\bfxi_1\rangle = n\sigma_1^2 \quad\text{and}\quad \langle\bfxi_1\bfxi_1^T\rangle = \sigma^2_1\bfI_n
\end{equation}
respectively. Let $\bfxi_2 \in \mathcal{R}^{m}$ be a second random \ac{iid} vector sampled from a zero-centred distribution with variance $\sigma^2_2$, then
\begin{equation}
  \langle\bfxi^T_1\bfxi_2\rangle = 0 \quad\text{and}\quad \langle\bfxi_1\bfxi^T_2\rangle = \bfZero  \in \mathcal{R}^{m \times n}.
\end{equation} \label{app:eq:independent-random-vector-inner-and-outer-product}
We continue by deriving general forms for the expected value of products of random matrices.
Let $\bfXi_1 \in \mathcal{R}^{m \times n}$ be a matrix whose entries are \ac{iid}, drawn from a zero-centred random distribution with variance $\sigma_1^2$.
Then,
\begin{equation}
  \langle\bfXi_1\rangle = \bfZero \in \mathcal{R}^{m \times n}.
\end{equation}
Further, let $\bfXi_{1i}$ and $\bfXi_{1j}$ denote the $i$-th and $j$-th column of the matrix.
Then it follows from \cref{app:eq:random-vector-inner-and-outer-product} that the expected value of the inner product for each column of $\bfXi_1$ is
\begin{equation}
  \la \bfXi_{1i}^T\bfXi_{1j}^{} \ra = \begin{cases}
    0,            & \text{if $i\neq j$} \\
    m \sigma_1^2, & \text{otherwise}
  \end{cases},
\end{equation}
and therefore that
\begin{equation} \label{app:eq:xi1T-xi1}
  \la \bfXi_1^T\bfXi_1 \ra = m \sigma_1^2 \bfI_n.
\end{equation}
Similarly, from \cref{app:eq:random-vector-inner-and-outer-product} it follow that
\begin{equation}
  \la \bfXi_{1i}^{}\bfXi_{1j}^T \ra = \begin{cases}
    \bfZero,        & \text{if $i\neq j$} \\
    \sigma_1^2\bfI, & \text{otherwise}
  \end{cases}
\end{equation}
and therefore that
\begin{equation} \label{app:eq:xi1-xi1T}
  \la \bfXi_1^{}\bfXi_1^T \ra = n\sigma_1^2\bfI_m.
\end{equation}
Let $\bfXi_2 \in \mathbb{R}^{l \times n}$ be a second random \ac{iid} matrix, sampled from a zero-centred distribution with variance $\sigma^2_2$, then it follows from \cref{app:eq:independent-random-vector-inner-and-outer-product} that
\begin{equation}
  \la \bfXi_2^{}\bfXi_1^T \ra = \bfZero. \end{equation}
We proceed by deriving equalities for the expected values of random matrices and their interactions with arbitrary constant matrices.
For arbitrary constant matrix $\bfB \in \mathcal{R}^{m \times m}$ and $i \neq j$ we get
\begin{equation}
  \begin{aligned}
    \la \bfXi_1^T \bfB \bfXi_1^{} \ra_{i,j} & = \tr\left(\bfB \la \bfXi_{1_j}^{}\bfXi_{1_i}^T \ra\right) \\
                                            & = \tr\left(\bfB \bfZero\right)                             \\
                                            & = 0
  \end{aligned}
\end{equation}
and for $i = j$
\begin{equation}
  \begin{aligned}
    \la \bfXi_1^T \bfB \bfXi_1^{} \ra_{i,j} & = \tr\left(\bfB \la \bfXi_{1_j}^{}\bfXi_{1_i}^T \ra\right) \\
                                            & = \tr\left(\bfB \sigma_1^2\bfI\right)                      \\
                                            & = \sigma_1^2\tr\left(\bfB\right)
  \end{aligned}
\end{equation}
and therefore
\begin{equation} \label{app:eq:xi1-b-xi1}
  \la \bfXi_1^T \bfB \bfXi_1^{} \ra = \sigma_1^2\tr\left(\bfB\right)\bfI.
\end{equation}

\section{The general linear solution} \label{app:sec:general-linear-solution}
From Theorem~3 in \textcite{laurent2018deep} it follows that any minimum of the convex and differentiable mean-squared error
\begin{equation}
  \mathcal{L}_\text{MSE} = \frac{1}{2P}\sum_{n=1}^P||\wb\wa\bfx_n - \bfy_n||_2^2
\end{equation}
corresponds to the global optimum of the convex single-layer optimisation problem
\begin{equation}
  \mathcal{L}_\text{MSE} = \frac{1}{2P}\sum_{n=1}^P||\wbar\bfx_n - \bfy||_2^2,
\end{equation}
as long as the underlying network
\begin{equation}
  \wbar = \wbo\wao,
\end{equation}
has no bottle-necks. The global optima of the network then correspond to
\begin{equation}
  \begin{aligned}
                    &  & \frac{\partial \mathcal{L}_\text{MSE}}{\partial\wbar}                              & = 0       \\
    \Leftrightarrow &  & \frac{1}{P}\sum_{n=1}^P\left(\wbar\bfx_n - \bfy_n\right)\bfx_n^T                   & = 0       \\
    \Leftrightarrow &  & \wbar\frac{1}{P}\sum_{n=1}^P\bfx_n\bfx_n^T - \frac{1}{P}\sum_{n=1}^P\bfy_n\bfx_n^T & = 0       \\
    \Leftrightarrow &  & \wbar\sigxx                                                                        & = \sigyx.
  \end{aligned}
\end{equation}
Resubstitution then gives the general linear solution
\begin{equation}
  \wbo\wao\sigxx = \sigyx.
\end{equation}
Finally, we note that the \ac{gls} can also be written as
\begin{equation} \label{app:eq:gls-rewritten}
  \begin{aligned}
                    &  & \wbo\wao\sigxx & = \sigyx                                                                 \\
    \Leftrightarrow &  & \wbo\wao       & = \sigyx\sigxx^+ + \bfZ(\bfI - \sigxx\sigxx^+)                           \\
    \Leftrightarrow &  & \wbo\wao       & = \sigyx\sigxx^+ + \bfZ(\bfI - 1/P\bfA\bfB^2\bfA^T P\bfA\bfB^{-2}\bfA^T) \\
    \Leftrightarrow &  & \wbo\wao       & = \bfU\bfS\bfV^T + \bfZ(\bfI - \bfA\bfA^T),
  \end{aligned}
\end{equation}
where $\bfZ \in \mathcal{R}^{N_o\times N_i}$ is an arbitrary matrix.

\begin{proof}[Proof of \cref{the:general-linear-solution}]
  In the following, we derive a complete parametrisation of the \ac{gls}
  \begin{equation}
    \wbo\wao = \bfU\bfS\bfV^T + \bfZ(\bfI - \bfA\bfA^T).
  \end{equation}
  We begin by rewriting $\wao$ in the basis of relevant, irrelevant and unobserved null directions
  \begin{equation}
    \wao = \wao\bfP_\text{r} + \wao\bfP_\text{i} + \wao\bfP_\text{u}
  \end{equation}
  with corresponding projectors $\bfP_\text{r} = \bfV\bfV^T$, $\bfP_\text{i} = \bfA\bfA^T - \bfV\bfV^T$, and $\bfP_\text{u} = \bfI - \bfA\bfA^T$.
  Substitution into the \ac{gls} (\cref{app:eq:gls-rewritten}) than reveals that directions that lie in the relevant input space must obey
  \begin{equation}
    \begin{aligned}
                      &  & \wbo(\wao\bfP_\text{r} + \wao\bfP_\text{i} + \wao\bfP_\text{u})\bfP_\text{r} & = (\bfU\bfS\bfV^T + \bfZ\bfP_\text{u})\bfP_\text{r} \\
      \Leftrightarrow &  & \wbo\wao\bfP_\text{r}                                                        & = \bfU\bfS\bfV^T,
    \end{aligned}
  \end{equation}
  from which it follows that $\wao\bfP_\text{r} = \wao\bfV\bfV^T$ can be parametrised as $\bfQ\sqrt{\bfS}\bfV^T$, where $\bfQ \in \mathcal{R}^{N_h \times r}$ is any full-column-rank matrix, projecting relevant inputs into hidden space.
  It then follows that we can further separate $\wbo\bfP_\text{h}$ into two subspaces, one projecting relevant hidden representations and one projecting all other dimensions
  \begin{equation}
    \begin{aligned}
                      &  & \wbo\wao\bfP_\text{r}                                           & = \bfU\bfS\bfV^T   \\
      \Leftrightarrow &  & (\wbo\bfQ\bfQ^+ + \wbo(\bfI - \bfQ\bfQ^+))\bfQ\sqrt{\bfS}\bfV^T & = \bfU\bfS\bfV^T   \\
      \Leftrightarrow &  & \wbo\bfQ\bfQ^+\bfQ\sqrt{\bfS}\bfV^T                             & = \bfU\bfS\bfV^T   \\
      \Leftrightarrow &  & \wbo\bfQ                                                        & = \bfU\sqrt{\bfS},
    \end{aligned}
  \end{equation}
  and thus $\wbo\bfQ\bfQ^+$ can be parametrised as $\bfU\sqrt{\bfS}\bfQ^+$.
  We call these two parts of the network function the core.
  We continue by analysing the subspace covered by irrelevant inputs
  \begin{equation}
    \begin{aligned}
                      &  & (\wbo\bfQ\bfQ^+ + \wbo(\bfI - \bfQ\bfQ^+))(\wao\bfP_\text{r} + \wao\bfP_\text{i} + \wao\bfP_\text{u})\bfP_\text{i} & = (\bfU\bfS\bfV^T + \bfZ\bfP_\text{n})\bfP_\text{i} \\
      \Leftrightarrow &  & (\wbo\bfQ\bfQ^+ + \wbo(\bfI - \bfQ\bfQ^+))\wao\bfP_\text{i}                                                        & = \bfZero                                           \\
      \Leftrightarrow &  & \wbo(\bfI - \bfQ\bfQ^+)\wao\bfP_\text{i}                                                                           & = -\wbo\bfQ\bfQ^+\wao\bfP_\text{i}.\end{aligned}
  \end{equation}
  This implies that task-irrelevant inputs can be mapped into the core readout space by input weights (right-hand side of the equation), provided that other task-irrelevant components are simultaneously projected in such a way that they cancel the effect (left-hand side of the equation).
  This ensure that the network output remains unchanged despite the first-layer weights processing task-irrelevant components.
  For the compensation in the two-layer network to be successful, we have to have
  \begin{equation}
    \begin{aligned}
                      &  & \wbo(\bfI - \bfQ\bfQ^+)\wao\bfP_\text{i}                    & = -\wbo\bfQ\bfQ^+\wao\bfP_\text{i}                                                                                                        \\
      \Leftrightarrow &  & \wbo(\bfI - \bfQ\bfQ^+)(\bfI - \bfQ\bfQ^+)\wao\bfP_\text{i} & = -\wbo\bfQ\bfQ^+\wao\bfP_\text{i}                                                                                                        \\
      \Leftrightarrow &  & \wbo(\bfI - \bfQ\bfQ^+)                                     & = -\wbo\bfQ\bfQ^+\wao\bfP_\text{i}\left[(\bfI - \bfQ\bfQ^+)\wao\bfP_\text{i}\right]^+                                                     \\
                      &  &                                                             & \qquad\qquad+ \tilde{\bfZ}\left[\bfI - (\bfI-\bfQ\bfQ^+)\wao\bfP_\text{i}((\bfI-\bfQ\bfQ^+)\wao\bfP_\text{i})^+\right](\bfI - \bfQ\bfQ^+) \\
      \Leftrightarrow &  & \wbo(\bfI - \bfQ\bfQ^+)                                     & = -\wbo\bfQ\bfQ^+\wao\bfP_\text{i}\left[(\bfI - \bfQ\bfQ^+)\wao\bfP_\text{i}\right]^+                                                     \\
                      &  &                                                             & \qquad\qquad+ \tilde{\bfZ}\left[\bfQ\bfQ^+ + (\bfI - \bfH\bfH^+)\right](\bfI - \bfQ\bfQ^+)                                                \\
      \Leftrightarrow &  & \wbo(\bfI - \bfQ\bfQ^+)                                     & = -\wbo\bfQ\bfQ^+\wao\bfP_\text{i}\left[(\bfI - \bfQ\bfQ^+)\wao\bfP_\text{i}\right]^+ + \tilde{\bfZ}(\bfI - \bfH\bfH^+),
    \end{aligned}
  \end{equation}
  where $\tilde{\bfZ} \in \mathcal{R}^{N_o\times N_h}$ is an arbitrary matrix. Crucially, the pseudo-inverse only exists if
  \begin{equation}
    \image(\bfQ\bfQ^+\wao\bfP_\text{i}) \subseteq \image((\bfI - \bfQ\bfQ^+)\wao\bfP_\text{i}).
  \end{equation}
  As both sides live in the same subspace $\bfP_\text{i}$, this is equivalent to
  \begin{equation}
    \rank(\bfQ\bfQ^+\wao\bfP_\text{i}) \leq \rank((\bfI - \bfQ\bfQ^+)\wao\bfP_\text{i}).
  \end{equation}
  Or in words, there must be at least as many dimensions of the irrelevant input space that are projected outside the core as there are dimensions that are projected into the core.
  Substitution then gives
  \begin{equation}
    \wbo(\bfI - \bfQ\bfQ^+) = -\bfU\sqrt{\bfS}\bfQ^+\wao\bfP_\text{i}\left[(\bfI - \bfQ\bfQ^+)\wao\bfP_\text{i}\right]^+ + \tilde{\bfZ}(\bfI - \bfH\bfH^+).
  \end{equation}
  Finally, we analyse the unoccupied null directions
  \begin{equation}
    \begin{aligned}
                      &  & \wbo(\wao\bfP_\text{r} + \wao\bfP_\text{i} + \wao\bfP_\text{u})\bfP_\text{u} & = (\bfU\bfS\bfV^T + \bfZ\bfP_\text{u})\bfP_\text{u} \\
      \Leftrightarrow &  & \wbo\wao\bfP_\text{u}                                                        & = \bfZ\bfP_\text{u},
    \end{aligned}
  \end{equation}
  from which it follows that $\wao\bfP_\text{u}$ can be chosen arbitrarily.
  In summary, we then have
  \begin{equation}
    \begin{aligned}
      \wao & = \wao\bfP_\text{r} + \wao\bfP_\text{i} + \wao\bfP_\text{u}                 \\
           & = \bfQ\sqrt{\bfS}\bfV^T + \bfGamma_1\bfP_\text{i} + \bfGamma_2\bfP_\text{u}
    \end{aligned}
  \end{equation}
  where $\bfGamma_1$, $\bfGamma_2 \in \mathcal{R}^{N_h\times N_i}$ can be chose freely up to the constraint that $\rank(\bfQ\bfQ^+\bfGamma_1\bfP_\text{i}) \leq \rank((\bfI - \bfQ\bfQ^+)\bfGamma_1\bfP_\text{i})$, and
  \begin{equation}
    \begin{aligned}
      \wbo & = \wbo\bfQ\bfQ^+ + \wbo(\bfI - \bfQ\bfQ^+)                        \\
           & = \bfU\sqrt{\bfS}\bfQ^+ + \bfPsi + \bfGamma_3(\bfI - \bfH\bfH^+),
    \end{aligned}
  \end{equation}
  where $\bfGamma_3 \in \mathcal{R}^{N_o\times N_h}$ is an arbitrary matrix and $\bfPsi = -\bfU\sqrt{\bfS}\bfQ^+\wao\bfP_\text{i}\left[(\bfI - \bfQ\bfQ^+)\wao\bfP_\text{i}\right]^+$.
\end{proof}

\section{Partitioning of the solution manifold} \label{app:sec:partitioning-of-suolution-manifold}
\begin{proof}[Proof of \cref{the:least-squares-solution}]
  We use the method of Lagrange multipliers to minimise
  \begin{equation}
    \argmin_{\wa, \wb}||\wb\wa||_F^2
  \end{equation}
  under the constraint that
  \begin{equation}
    \wb\wa\sigxx = \sigyx.
  \end{equation}
  We begin by substituting $\wobar = \wb\wa$.
  Then the Lagrangian is
  \begin{equation}
    \mathcal{L} = ||\wobar||_F^2 + \tr\left(\bfLambda^T\left(\wobar\sigxx - \sigyx\right)\right)
  \end{equation}
  with gradients
  \begin{equation} \label{app:eq:least-squares-solution-partial-w1}
    \begin{aligned}
                      &  & \frac{\partial \mathcal{L}}{\partial \wobar} & = 2\wobar + \bfLambda\sigxx  = 0 \\
      \Leftrightarrow &  & \wobar                                       & = -\frac{1}{2}\bfLambda\sigxx,
    \end{aligned}
  \end{equation}
  and
  \begin{equation} \label{app:eq:least-sqaures-solution-partial-lambda}
    \frac{\partial \mathcal{L}}{\partial \bfLambda} = \wobar\sigxx - \sigyx = 0.
  \end{equation}
  Then starting from \cref{app:eq:least-squares-solution-partial-w1} we have
  \begin{equation}
    \begin{aligned}
                      &  & \wobar                      & = -\frac{1}{2}\bfLambda\sigxx                                                 \\
      \Leftrightarrow &  & \wobar\sigxx                & = -\frac{1}{2}\bfLambda\sigxx\sigxx                                           \\
      \Leftrightarrow &  & \sigyx                      & = -\frac{1}{2}\bfLambda\frac{1}{P}\bfA\bfB^2\bfA^T\frac{1}{P}\bfA\bfB^2\bfA^T \\
      \Leftrightarrow &  & \sigyx P\bfA\bfB^{-2}\bfA^T & = -\frac{1}{2}\bfLambda\frac{1}{P}\bfA\bfB^2\bfA^T                            \\
      \Leftrightarrow &  & \sigyx(\sigxx)^+            & = -\frac{1}{2}\bfLambda\sigxx,
    \end{aligned}
  \end{equation}
  where in the third step we substituted \cref{app:eq:least-sqaures-solution-partial-lambda}. Then, resubstitution into \cref{app:eq:least-squares-solution-partial-w1} yields
  \begin{equation}
    \wobar = \wbo\wao =  \sigyx(\sigxx)^+ = \bfU\bfS\bfV^T.
  \end{equation}
  Next, we derive a complete parametrisation of the \ac{lss}
  \begin{equation}
    \wbo\wao = \bfU\bfS\bfV^T.
  \end{equation}
  Again, as in the parametrisation for the \ac{gls} (\cref{the:general-linear-solution}), we rewrite $\wao$ in the basis of relevant, irrelevant and unobserved null directions
  \begin{equation}
    \wao = \wao\bfP_\text{r} + \wao\bfP_\text{i} + \wao\bfP_\text{u}
  \end{equation}
  with corresponding projectors $\bfP_\text{r} = \bfV\bfV^T$, $\bfP_\text{i} = \bfA\bfA^T - \bfV\bfV^T$, and $\bfP_\text{u} = \bfI - \bfA\bfA^T$.
  First, we note that the equations for the relevant and irrelevant input directions are identical to the ones presented in~\cref{the:general-linear-solution}, that is
  \begin{equation}
    \begin{aligned}
                      &  & \wbo(\wao\bfP_\text{r} + \wao\bfP_\text{i} + \wao\bfP_\text{u})\bfP_\text{r} & = \bfU\bfS\bfV^T\bfP_\text{r} \\
      \Leftrightarrow &  & \wbo\wao\bfP_\text{r}                                                        & = \bfU\bfS\bfV^T,
    \end{aligned}
  \end{equation}
  and
  \begin{equation}
    \begin{aligned}
                      &  & (\wbo\bfQ\bfQ^+ + \wbo(\bfI - \bfQ\bfQ^+))(\wao\bfP_\text{r} + \wao\bfP_\text{i} + \wao\bfP_\text{u})\bfP_\text{i} & = \bfU\bfS\bfV^T\bfP_\text{i}       \\
      \Leftrightarrow &  & \wbo(\bfI - \bfQ\bfQ^+)\wao\bfP_\text{i}                                                                           & = -\wbo\bfQ\bfQ^+\wao\bfP_\text{i}.
    \end{aligned}
  \end{equation}
  From which it follows that
  \begin{align}
    \wao\bfP_\text{r}       & = \bfQ\sqrt{\bfS}\bfV^T                     \\
    \wao\bfP_\text{i}       & = \bfGamma_1\bfP_\text{i}                   \\
    \wbo\bfQ\bfQ^+          & = \bfU\sqrt{\bfS}\bfQ^+                     \\
    \wbo(\bfI - \bfQ\bfQ^+) & = \bfPsi + \tilde{\bfZ}(\bfI - \bfH\bfH^+),
  \end{align}
  where $\bfQ \in \mathcal{R}^{N_h \times r}$ is an arbitrary full-column-rank matrix, $\bfGamma_1$, is an arbitrary matrix subject to the constraint $\rank(\bfQ\bfQ^+\bfGamma_1\bfP_\text{i}) \leq \rank((\bfI - \bfQ\bfQ^+)\bfGamma_1\bfP_\text{i})$, $\bfPsi = -\bfU\sqrt{\bfS}\bfQ^+\bfGamma_1\bfP_\text{i}\left[(\bfI - \bfQ\bfQ^+)\bfGamma_1\bfP_\text{i}\right]^+$, and $\tilde{\bfZ} \in \mathcal{R}^{N_o \times N_h}$ is an arbitrary matrix.
  However, \ac{lss} differ in how they constrain the solution manifold in the case of unobserved null directions.
  In particular, $\hat{\bfZ}(\bfI - \bfH\bfH^+)$ is constrained by
  \begin{equation}
    \begin{aligned}
                      &  & \wbo(\wao\bfP_\text{r} + \wao\bfP_\text{i} + \wao\bfP_\text{u})\bfP_\text{u}        & = \bfU\bfS\bfV^T\bfP_\text{u}                                                                                                            \\
      \Leftrightarrow &  & (\wbo\bfQ\bfQ^+ + \wbo(\bfI - \bfQ\bfQ^+))\wao\bfP_\text{u}                         & = \bfZero                                                                                                                                \\
      \Leftrightarrow &  & (\bfU\sqrt{\bfS}\bfQ^+ + \bfPsi + \tilde{\bfZ}(\bfI - \bfH\bfH^+))\wao\bfP_\text{u} & = \bfZero                                                                                                                                \\
      \Leftrightarrow &  & \tilde{\bfZ}(\bfI - \bfH\bfH^+)(\bfI - \bfH\bfH^+)\wao\bfP_\text{u}                 & = -(\bfU\sqrt{\bfS}\bfQ^+ + \bfPsi)\wao\bfP_\text{u}                                                                                     \\
      \Leftrightarrow &  & \tilde{\bfZ}(\bfI - \bfH\bfH^+)                                                     & = -(\bfU\sqrt{\bfS}\bfQ^+ + \bfPsi)\wao\bfP_\text{u}\left[(\bfI - \bfH\bfH^+)\wao\bfP_\text{u}\right]^+                                  \\
                      &  &                                                                                     & \qquad\ + \hat{\bfZ}[\bfI - (\bfI - \bfH\bfH^+)\wao\bfP_\text{u}((\bfI - \bfH\bfH^+)\wao\bfP_\text{u})^+](\bfI - \bfH\bfH^+)             \\
      \Leftrightarrow &  & \tilde{\bfZ}(\bfI - \bfH\bfH^+)                                                     & = -(\bfU\sqrt{\bfS}\bfQ^+ + \bfPsi)\wao\bfP_\text{u}\left[(\bfI - \bfH\bfH^+)\wao\bfP_\text{u}\right]^+                                  \\
                      &  &                                                                                     & \qquad\ + \hat{\bfZ}[\bfH\bfH^+ + (\bfI - \wao\wao^+)](\bfI - \bfH\bfH^+)                                                                \\
      \Leftrightarrow &  & \tilde{\bfZ}(\bfI - \bfH\bfH^+)                                                     & = -(\bfU\sqrt{\bfS}\bfQ^+ + \bfPsi)\wao\bfP_\text{u}\left[(\bfI - \bfH\bfH^+)\wao\bfP_\text{u}\right]^+ + \hat{\bfZ}(\bfI - \wao\wao^+),
    \end{aligned}
  \end{equation}
  where $\hat{\bfZ} \in \mathcal{R}^{N_o\times N_h}$ is an arbitrary matrix.
  Again, for the pseudo-inverse to exist we must have
  \begin{equation}
    \image(\bfH\bfH^+\wao\bfP_\text{u}) \subseteq \image((\bfI - \bfH\bfH^+)\wao\bfP_\text{u}).
  \end{equation}
  As both sides live in the same subspace $\bfP_\text{u}$, this is equivalent to
  \begin{equation}
    \rank(\bfH\bfH^+\wao\bfP_\text{u}) \leq \rank((\bfI - \bfH\bfH^+)\wao\bfP_\text{u}).
  \end{equation}
  In other words, at least as many unobserved dimensions that are projected into the occupied hidden-layer dimensions have to be projected into the unoccupied hidden-layer dimensions in order for a correction to exist.
  In summary, we then have
  \begin{equation}
    \begin{aligned}
      \wao & = \wao\bfP_\text{r} + \wao\bfP_\text{i} + \wao\bfP_\text{u}                 \\
           & = \bfQ\sqrt{\bfS}\bfV^T + \bfGamma_1\bfP_\text{i} + \bfGamma_2\bfP_\text{u}
    \end{aligned}
  \end{equation}
  where $\bfGamma_1$, $\bfGamma_2 \in \mathcal{R}^{N_h\times N_i}$ can be chose freely up to the constraints that $\rank(\bfQ\bfQ^+\bfGamma_1\bfP_\text{i}) \leq \rank((\bfI - \bfQ\bfQ^+)\bfGamma_1\bfP_\text{i})$, and $\rank(\bfH\bfH^+\wao\bfP_\text{u}) \leq \rank((\bfI - \bfH\bfH^+)\wao\bfP_\text{u})$, and
  \begin{equation}
    \begin{aligned}
      \wbo & = \wbo\bfQ\bfQ^+ + \wbo(\bfI - \bfQ\bfQ^+)                                 \\
           & = \bfU\sqrt{\bfS}\bfQ^+ + \bfPsi + \bfPhi + \bfGamma_3(\bfI - \wao\wao^+),
    \end{aligned}
  \end{equation}
  where $\bfGamma_3 \in \mathcal{R}^{N_o\times N_h}$ is an arbitrary matrix, $\bfPsi = -\bfU\sqrt{\bfS}\bfQ^+\wao\bfP_\text{i}\left[(\bfI - \bfQ\bfQ^+)\wao\bfP_\text{i}\right]^+$, and $\bfPhi = -(\bfU\sqrt{\bfS}\bfQ^+ + \bfPsi)\wao\bfP_\text{u}\left[(\bfI - \bfH\bfH^+)\wao\bfP_\text{u}\right]^+$.
\end{proof}

\begin{proof}[Proof of \cref{the:minimum-representation-norm-solution}]
  We use the method of Lagrange multipliers to minimise
  \begin{equation}
    \argmin_{\wa, \wb}||\wa\bfX||_F^2 + ||\wb||_F^2
  \end{equation}
  under the constraint that
  \begin{equation}
    \wb\wa\sigxx = \sigyx.
  \end{equation}
  Then the Lagrangian is
  \begin{equation}
    \begin{aligned}
      \mathcal{L} =  \ ||\wa\bfX||_F^2 + ||\wb||_F^2
      + \tr\left(\bfLambda^T\left(\wb\wa\sigxx - \sigyx\right)\right)
    \end{aligned}
  \end{equation}
  with gradients
  \begin{equation} \label{app:eq:minimum-weight-solution-partial-w1}
    \begin{aligned}
                      &  & \frac{\partial \mathcal{L}}{\partial \wa} & = 2\wa\bfX\bfX^T + \wb^T\bfLambda\sigxx = 0 \\
      \Leftrightarrow &  & \wa\bfX\bfX^T                             & = -\frac{1}{2}\wb^T\bfLambda\sigxx,
    \end{aligned}
  \end{equation}
  \begin{equation} \label{app:eq:minimum-weight-solution-partial-w2}
    \begin{aligned}
                      &  & \frac{\partial \mathcal{L}}{\partial \wb} & = 2\wb + \bfLambda\sigxx\wa^T = 0   \\
      \Leftrightarrow &  & \wb                                       & = -\frac{1}{2}\bfLambda\sigxx\wa^T,
    \end{aligned}
  \end{equation}
  and
  \begin{equation} \label{app:eq:minimum-weight-solution-partial-lambda}
    \frac{\partial \mathcal{L}}{\partial \bfLambda} = \wb\wa\sigxx - \sigyx = 0.
  \end{equation}
  Starting from \cref{app:eq:minimum-weight-solution-partial-w2}, we then have
  \begin{equation}
    \begin{aligned}
                      &  & \wb      & = -\frac{1}{2}\bfLambda\sigxx\wa^T      \\
      \Leftrightarrow &  & \wb^T\wb & = -\frac{1}{2}\wb^T\bfLambda\sigxx\wa^T \\
      \Leftrightarrow &  & \wb^T\wb & = \wa\bfX\bfX^T\wa^T,
    \end{aligned}
  \end{equation}
  where in the last line we substituted \cref{app:eq:minimum-weight-solution-partial-w1}. Let
  \begin{equation}
    \svd(\wb) = \bfA\bfB\bfC^T\ \text{and}\ \svd(\wa\bfX) = \bfD\bfE\bfF^T
  \end{equation}
  be the \ac{csvd} of the network weights, then
  \begin{equation}
    \begin{aligned}
                      &  & \wb^T\wb                     & = \wa\bfX\bfX^T\wa^T           \\
      \Leftrightarrow &  & \bfC\bfB\bfA^T\bfA\bfB\bfC^T & = \bfD\bfE\bfF^T\bfF\bfE\bfD^T \\
      \Leftrightarrow &  & \bfC\bfB^2\bfC^T             & = \bfD\bfE^2\bfD^T.
    \end{aligned}
  \end{equation}
  Since $\bfC$ and $\bfD$ are (semi-)orthonormal matrices and $\bfB$ and $\bfE$ are diagonal matrices with strictly positive entries it follows that
  \begin{equation}
    \bfC = \bfD\ \ \text{and}\ \ \bfB^2 = \bfE^2 \Leftrightarrow \bfB = \bfE.
  \end{equation}
  In the following we denote $\bfC$ and $\bfD$ as $\bfR$, and $\bfB$ and $\bfE$ as $\bfG$ and write
  \begin{equation}
    \wa\bfX = \bfR\bfG\bfF^T\ \text{and}\ \wb = \bfA\bfG\bfR^T,
  \end{equation}
  where $\bfR$ is an arbitrary (semi-)orthogonal matrix. Finally, let
  \begin{equation}
    \svd(\bfX) = \bfJ\bfK\bfL^T,
  \end{equation}
  and
  \begin{equation}
    \begin{aligned}
      \svd(\bfY\bfX^T\bfX^{+T}) & = \svd(\bfY\bfL\bfK\bfJ^T\bfJ\bfK^{-1}\bfL^T) \\
                                & = \svd(\bfY\bfL\bfL^T)                        \\
                                & = \bfM\bfN\bfO^T.
    \end{aligned}
  \end{equation}
  Then starting from \cref{app:eq:minimum-weight-solution-partial-lambda} we get
  \begin{equation}
    \begin{aligned}
                      &  & \wb\wa\sigxx                 & = \sigyx             \\
      \Leftrightarrow &  & \wb\wa\bfX\bfX^T             & = \bfY\bfX^T         \\
      \Leftrightarrow &  & \wb\wa\bfJ\bfK^2\bfJ^T       & = \bfY\bfL\bfK\bfJ^T \\
      \Leftrightarrow &  & \wb\wa\bfJ\bfK\bfL^T         & = \bfY\bfL\bfL^T     \\
      \Leftrightarrow &  & \bfA\bfG\bfR^T\bfR\bfG\bfF^T & = \bfM\bfN\bfO^T     \\
      \Leftrightarrow &  & \bfA\bfG^2\bfF^T             & = \bfM\bfN\bfO^T,
    \end{aligned}
  \end{equation}
  from which it follows that
  \begin{equation}
    \bfA = \bfM,\ \bfF^T = \bfO^T,\ \text{and}\ \bfG^2 = \bfN \Leftrightarrow \bfG = \sqrt{\bfN}
  \end{equation}
  and therefore that
  \begin{equation}
    \wb = \bfM\sqrt{\bfN}\bfR^T\ \text{and}\ \wa\bfX = \bfR\sqrt{\bfN}\bfO^T.
  \end{equation}
  Finally, we rewrite
  \begin{equation}
    \begin{aligned}
      \wa\bfX           & = \bfR\sqrt{\bfN}\bfO^T                                                                     \\
      \wa\bfJ\bfK\bfL^T & = \bfR\sqrt{\bfN}\bfO^T                                                                     \\
      \wa               & = \bfR\sqrt{\bfN}\bfO^T\bfL\bfK^{-1}\bfJ^T + \bfZ\left(\bfI - \bfJ\bfJ^T\right)             \\
      \wa               & = \bfR\sqrt{\bfN}\bfO^T\bfX^+ + \underbrace{\bfZ\left(\bfI - \bfJ\bfJ^T\right)}_{\bfGamma},
    \end{aligned}
  \end{equation}
  where $\bfZ$ is an arbitrary matrix.
\end{proof}

\begin{proof}[Proof of \cref{the:minimum-weight-norm-solution}]
  We use the method of Lagrange multipliers to minimise
  \begin{equation}
    \argmin_{\wa, \wb}||\wa||_F^2 + ||\wb||_F^2
  \end{equation}
  under the constraint that
  \begin{equation}
    \wb\wa\sigxx = \sigyx.
  \end{equation}
  Then the Lagrangian is
  \begin{equation}
    \begin{aligned}
      \mathcal{L} = & \ ||\wa||_F^2 + ||\wb||_F^2                                    + \tr\left(\bfLambda^T\left(\wb\wa\sigxx - \sigyx\right)\right)
    \end{aligned}
  \end{equation}
  with gradients
  \begin{equation} \label{app:eq:minimum-norm-solution-partial-w1}
    \begin{aligned}
       &  & \frac{\partial \mathcal{L}}{\partial \wa} = 2\wa + \wb^T\bfLambda\sigxx & = 0                                 \\
       &  & \Leftrightarrow \wa                                                     & = -\frac{1}{2}\wb^T\bfLambda\sigxx,
    \end{aligned}
  \end{equation}
  \begin{equation} \label{app:eq:minimum-norm-solution-partial-w2}
    \begin{aligned}
       &  & \frac{\partial \mathcal{L}}{\partial \wb} = 2\wb + \bfLambda\sigxx\wa^T & = 0                                 \\
       &  & \Leftrightarrow \wb                                                     & = -\frac{1}{2}\bfLambda\sigxx\wa^T,
    \end{aligned}
  \end{equation}
  and
  \begin{equation} \label{app:eq:minimum-norm-solution-partial-lambda}
    \frac{\partial \mathcal{L}}{\partial \bfLambda} = \wb\wa\sigxx - \sigyx = 0.
  \end{equation}
  Starting from \cref{app:eq:minimum-norm-solution-partial-w2} we then have
  \begin{equation}
    \begin{aligned}
                      &  & \wb      & = -\frac{1}{2}\bfLambda\sigxx\wa^T      \\
      \Leftrightarrow &  & \wb^T\wb & = -\frac{1}{2}\wb^T\bfLambda\sigxx\wa^T \\
      \Leftrightarrow &  & \wb^T\wb & = \wa\wa^T,
    \end{aligned}
  \end{equation}
  where in the last line we substituted \cref{app:eq:minimum-norm-solution-partial-w1}. Let
  \begin{equation}
    \svd(\wb) = \bfA\bfB\bfC^T\ \text{and}\ \svd(\wa) = \bfD\bfE\bfF^T
  \end{equation}
  be the \ac{csvd} of the network weights, then
  \begin{equation}
    \begin{aligned}
                      &  & \wb^T\wb                     & = \wa\wa^T                     \\
      \Leftrightarrow &  & \bfC\bfB\bfA^T\bfA\bfB\bfC^T & = \bfD\bfE\bfF^T\bfF\bfE\bfD^T \\
      \Leftrightarrow &  & \bfC\bfB^2\bfC^T             & = \bfD\bfE^2\bfD^T.
    \end{aligned}
  \end{equation}
  Since $\bfC$ and $\bfD$ are (semi-)orthonormal matrices and $\bfB$ and $\bfE$ are diagonal matrices with strictly positive entries it follows that
  \begin{equation}
    \bfC = \bfD\ \text{and}\ \bfB^2 = \bfE^2 \Leftrightarrow \bfB = \bfE.
  \end{equation}
  It further follows that $\wa$ and $\wb$ must be of identical rank.
  In the following we denote $\bfC$ and $\bfD$ as $\bfR$, and $\bfB$ and $\bfE$ as $\bfG$ and write
  \begin{equation}
    \wa = \bfR\bfG\bfF^T\ \text{and}\ \wb = \bfA\bfG\bfR^T,
  \end{equation}
  where $\bfR$ is an arbitrary (semi-)orthogonal matrix. Finally, let
  \begin{equation}
    \svd(\sigxx) = \bfJ\bfK\bfJ^T
  \end{equation}
  and
  \begin{equation}
    \svd(\sigyx\sigxx^+) = \bfU\bfS\bfV^T
  \end{equation}
  denote the \ac{csvd} of the input covariance matrix and a least-squares solution. Then, starting from \cref{app:eq:minimum-norm-solution-partial-lambda} and using \cref{app:eq:minimum-norm-solution-partial-w1}, we get
  \begin{equation}
    \begin{aligned}
                      &  & \wb\wa\sigxx                                  & = \sigyx                    \\
      \Leftrightarrow &  & -\frac{1}{2}\wb\wb^T\bfLambda\sigxx\sigxx     & = \sigyx                    \\
      \Leftrightarrow &  & -\frac{1}{2}\wb\wb^T\bfLambda\bfJ\bfK^2\bfJ^T & = \sigyx                    \\
      \Leftrightarrow &  & -\frac{1}{2}\wb\wb^T\bfLambda\bfJ\bfK\bfJ^T   & = \sigyx\bfJ\bfK^{-1}\bfJ^T \\
      \Leftrightarrow &  & -\frac{1}{2}\wb\wb^T\bfLambda\sigxx           & = \sigyx\sigxx^+            \\
      \Leftrightarrow &  & \wb\wa                                        & = \bfU\bfS\bfV^T            \\
      \Leftrightarrow &  & \bfA\bfG\bfR^T\bfR\bfG\bfF^T                  & = \bfU\bfS\bfV^T            \\
      \Leftrightarrow &  & \bfA\bfG^2\bfF^T                              & = \bfU\bfS\bfV^T,
    \end{aligned}
  \end{equation}
  from which it follows that
  \begin{equation}
    \bfA = \bfU\text{,}\ \bfF^T = \bfV^T\,\ \text{and}\ \bfG^2 = \bfS \Leftrightarrow \bfG = \sqrt{\bfS}
  \end{equation}
  and therefore that
  \begin{equation}
    \wb = \bfU\sqrt{\bfS}\bfR^T\ \text{and}\ \wa = \bfR\sqrt{\bfS}\bfV^T.
  \end{equation}
\end{proof}

\section{Hidden-layer representations} \label{app:sec:hidden-layer-representations}
\begin{proposition} \label{app:prop:gls-lss-identical-h}
  Let $\svd(\bfX) = \bfA\bfB\bfC^T$, then \ac{gls} are identical to \ac{lss} when operating on inputs $\bfX$ as any \ac{gls} (\cref{app:eq:gls-rewritten}) can be written as
  \begin{equation}
    \begin{aligned}
                      &  & \wbo\wao\bfX & = \sigyx(\sigxx)^+\bfX + \bfZ(\bfI - \bfA\bfA^T)\bfX           \\
      \Leftrightarrow &  & \wbo\wao\bfX & = \sigyx(\sigxx)^+\bfX + \bfZ(\bfI - \bfA\bfA^T)\bfA\bfB\bfC^T \\
      \Leftrightarrow &  & \wbo\wao\bfX & = \sigyx(\sigxx)^+\bfX + \bfZ(\bfA\bfB\bfC^T - \bfA\bfB\bfC^T) \\
      \Leftrightarrow &  & \wbo\wao\bfX & = \sigyx(\sigxx)^+\bfX.
    \end{aligned}
  \end{equation}
  As a consequence, hidden-layer representations of the training data in \ac{gls} and \ac{lss} can be analysed by means of studying the \ac{lss} only.
\end{proposition}

\begin{proof}[Proof of \cref{cor:lss-rsm}]
  Following ~\cref{app:prop:gls-lss-identical-h} and given the parametrisation of $\wao$ as derived in ~\cref{the:general-linear-solution,the:least-squares-solution}, we derive
  \begin{equation}
    \begin{aligned}
      \rsm & = \bfX^T\wao^T\wao\bfX                                                                                                                                                                                          \\
           & = \bfX^T(\bfV\sqrt{\bfS}\bfQ^T + \bfP_\text{i}^T\bfGamma_1^T + \bfP_\text{u}^T\bfGamma_2^T)(\bfQ\sqrt{\bfS}\bfV^T + \bfGamma_1\bfP_\text{i} + \bfGamma_2\bfP_\text{u})\bfX                                      \\
           & = \bfX^T(\bfV\sqrt{\bfS}\bfQ^T\bfQ\sqrt{\bfS}\bfV^T + \bfV\sqrt{\bfS}\bfQ^T\bfGamma_1\bfP_\text{i} + \bfP_\text{i}^T\bfGamma_1^T\bfQ\sqrt{\bfS}\bfV^T + \bfP_\text{i}^T\bfGamma_1^T\bfGamma_1\bfP_\text{i})\bfX
    \end{aligned}
  \end{equation}
\end{proof}
\begin{proof}[Proof of \cref{corr:mrns-rsm}]
  Let
  \begin{equation}
    \svd(\bfX) = \bfA\bfB\bfC^T,
  \end{equation}
  then given the subset of $\wao$ as defined in \cref{the:minimum-representation-norm-solution} we derive
  \begin{equation}
    \begin{aligned}
      \bfX^T(\bfX^{+T}\bfO\sqrt{\bfN}\bfR^T + \bfGamma^T)(\bfR\sqrt{\bfN}\bfO^T\bfX^+ + \bfGamma)\bfX
      =\  & \bfX^T\bfX^{+T}\bfO\sqrt{\bfN}\bfR^T\bfR\sqrt{\bfN}\bfO^T\bfX^+\bfX              \\
      =\  & \bfC\bfB\bfA^T\bfA\bfB^{-1}\bfC^T\bfO\bfN\bfO^T\bfC\bfB^{-1}\bfA^T\bfA\bfB\bfC^T \\
      =\  & \bfC\bfC^T\bfO\bfN\bfO^T\bfC\bfC^T                                               \\
      =\  & \bfO\bfN\bfO^T~,
    \end{aligned}
  \end{equation}
  where in the last step we used that
  \begin{equation}
    \bfY\bfX^T\bfX^{+T} = \bfY\bfC\bfC^T.
  \end{equation}
\end{proof}
\begin{proof}[Proof of \cref{corr:mwns-rsm}]
  Given the subset of $\wao$ as defined in \cref{the:minimum-weight-norm-solution} we derive
  \begin{equation}
    \begin{aligned}
      \bfX^T\bfV\sqrt{\bfS}\bfR^T\bfR\sqrt{\bfS}\bfV^T\bfX
      =\  & \bfX^T\bfV\bfS\bfV^T\bfX~.
    \end{aligned}
  \end{equation}
\end{proof}

\section{Secondary error} \label{app:sec:secondary-error}
\begin{proof} [Proof of \cref{the:secondary-error}]
  Let the secondary error be denoted by
  \begin{equation}
    \mathcal{G}_\text{MSE} = \frac{1}{2Q}\sum_{n=1}^Q||\wb\wa\bfxt_n - \bfyt_n||_2^2,
  \end{equation}
  where input vectors $\bfxt_n$ and corresponding target values $\bfyt_n$ come from a secondary dataset $\tilde{\mathcal{D}} = \{(\bfxt_n, \bfyt_n)\}_{n=1}^Q$ with a total of $Q$ input-output pairs.
  We want to find the point on the solution manifold of the primary task such that $\mathcal{G}_\text{MSE}$ is minimised.
  We therefore substitute the \ac{gls} (see~\cref{app:eq:gls-rewritten}) into the secondary error
  \begin{equation}
    \mathcal{G}_\text{MSE} = \frac{1}{2Q}\sum_{n=1}^Q\left|\left|\left(\sigyx\sigxx^+ + \bfZ\bfP_\text{u}\right)\bfxt_n - \bfyt_n\right|\right|_2^2,
  \end{equation}
  where $\bfZ \in \mathcal{R}^{N_o \times N_i}$ is an arbitrary matrix and $\bfP_\text{u} = \bfI - \bfA\bfA^T$.
  Thus $\bfZ\bfP_\text{u}$, which describes all possible transformations of the network function that lie in the unoccupied input space of the primary task is the only degree of freedom.
  Thus, we have to find $\tilde{\bfZ}$, which minimises the error.
  To this end, we define
  \begin{equation}
    \svd(\bfP_\text{u}\bfXt) = \bfD\bfE\bfF^T,
  \end{equation}
  and continue with
  \begin{equation}
    \begin{aligned}
                      &  & \frac{\partial \mathcal{G}_\text{MSE}}{\partial \tilde{\bfZ}} = \frac{1}{Q}\sum_{n=1}^Q\left(\left(\sigyx\sigxx^+ + \tilde{\bfZ}\bfP_\text{u}\right)\bfxt_n - \bfyt_n\right)\bfxt_n^T\bfP_\text{u} & = \bfZero                                                                                                          \\
      \Leftrightarrow &  & \left(\left(\sigyx\sigxx^+ + \tilde{\bfZ}\bfP_\text{u}\right)\sigxxt - \sigyxt\right)\bfP_\text{u}                                                                                                 & = \bfZero                                                                                                          \\
      \Leftrightarrow &  & \tilde{\bfZ}\bfP_\text{u}\sigxxt\bfP_\text{u}                                                                                                                                                      & = \left(\sigyxt - \sigyx\sigxx^+\sigxxt\right)\bfP_\text{u}                                                        \\
      \Leftrightarrow &  & \tilde{\bfZ}\bfP_\text{u}\bfXt\bfXt^T\bfP_\text{u}                                                                                                                                                 & = \left(\bfYt - \sigyx\sigxx^+\bfXt\right)\bfXt^T\bfP_\text{u}                                                     \\
      \Leftrightarrow &  & \tilde{\bfZ}\bfD\bfE^2\bfD^T                                                                                                                                                                       & = \left(\bfYt - \sigyx\sigxx^+\bfXt\right)\bfF\bfE\bfD^T                                                           \\
      \Leftrightarrow &  & \tilde{\bfZ}                                                                                                                                                                                       & = \left(\bfYt - \sigyx\sigxx^+\bfXt\right)\bfF\bfE^{-1}\bfD^T + \tilde{\bfGamma}(\bfI - \bfD\bfD^T)                \\
      \Leftrightarrow &  & \tilde{\bfZ}                                                                                                                                                                                       & = \left(\bfYt - \sigyx\sigxx^+\bfXt\right)\left(\bfP_\text{u}\bfXt\right)^+ + \tilde{\bfGamma}(\bfI - \bfD\bfD^T),
    \end{aligned}
  \end{equation}
  where $\tilde{\bfGamma} \in \mathcal{R}^{N_o \times N_i}$ is an arbitrary matrix.
  Note that $\tilde{\bfGamma}(\bfI - \bfD\bfD^T)\bfP_\text{u}$ describes all possible transformations that lie in the null spaces of both $\bfX$ and $\bfXt$, and therefore have no effect on the processing of either the primary or secondary data.
\end{proof}

\section{Input and parameter noise} \label{app:sec:noise-proofs}
\begin{proof}[Proof of \cref{the:parameter-and-input-noise}]
  Let $\bfxi_{\bfx_n}$ denote vectors and $\bfXi_1$ and $\bfXi_2$ denote matrices whose entries are \ac{iid} and drawn from a zero-centred random distribution with variance $\sigma^2_{\bfx}$, $\sigma^2_{\wao}$ and $\sigma^2_{\wbo}$ respectively. Then, we can write the expected mean-squared error under additive input noise and parameter noise as
  \begin{equation}
    \begin{aligned}
       & \bigg\langle\frac{1}{2P}\sum_{n=1}^P||\left(\wbo + \bfXi_2\right)\left(\wao + \bfXi_1\right)\left(\bfx_n + \bfxi_{\bfx_n}\right) - \bfy_n||_2^2\bigg\rangle \\
       & =\frac{1}{2P}\sum_{n=1}^P\bigg\langle||\underbrace{\wbo\wao\bfx_n - \bfy_n}_a + \underbrace{\big(\wbo\wao + \wb\bfXi_1 +}_b
      \underbrace{\bfXi_2\wa +\bfXi_2\bfXi_1\big) \bfxi_{\bfx_n} + \big(\wb\bfXi_1 + \bfXi_2\wa +}_b
      \underbrace{\bfXi_2\bfXi_2\big)\bfx_n}_b||_2^2\bigg\rangle                                                                                                     \\
       & =\frac{1}{2P}\sum_{n=1}^P \bfa^T\bfa + \frac{1}{2P}\sum_{n=1}^P2 \bfa^T\big\langle\bfb \big\rangle + \frac{1}{2P}\sum_{n=1}^P \la \bfb^T\bfb \ra.
    \end{aligned}
  \end{equation}
  In the following we solve each of the three summands independently.
  Using the definition of a linear solution (\cref{def:general-linear-solution}), we derive
  \begin{equation}
    \begin{aligned}
      \label{app:eq:noise-ata}
      \frac{1}{2P}\sum_{n=1}^P \bfa^T\bfa =\  & \frac{1}{2P}\sum_{n=1}^P \bfx_n^T\wao^T\wbo^T\wbo\wao\bfx_n - \frac{1}{2P}\sum_{n=1}^P \bfx_n^T\wao^T\wbo^T\bfy_n
      - \frac{1}{2P}\sum_{n=1}^P \bfy_n^T\wbo\wao\bfx_n + \frac{1}{2P}\sum_{n=1}^P \bfy_n^T\bfy_n                                                                 \\
      =\                                      & \frac{1}{2}\tr(\wao^T\wbo^T\wbo\wao\sigxx) - \tr(\wao^T\wbo^T\sigyx)
      + \frac{1}{2}\tr(\sigyy)                                                                                                                                    \\
      =\                                      & -\frac{1}{2}\tr(\wao^T\wbo^T\wbo\wao\sigxx(\sigxx)^+\sigxx^T) + \frac{1}{2}\tr(\sigyy)                            \\
      =\                                      & -\frac{1}{2}\tr(\sigyx(\sigxx)^+\sigyx^T) + \frac{1}{2}\tr(\sigyy)                                                \\
      =\                                      & c,
    \end{aligned}
  \end{equation}
  which is a noise-independent constant that only depends on the statistics of the training data.
  Next, using the assumption that the noise is zero-centred, we derive that
  \begin{equation} \label{app:eq:noise-atb}
    \frac{1}{2P}\sum_{n=1}^P2 \bfa^T\big\langle\bfb \big\rangle = 0
  \end{equation}
  using that
  \begin{align}
    \big\langle\bfb\big\rangle  = \Big\langle\big(\wbo\wao + \wbo\bfXi_1 + \bfXi_2\wao +\bfXi_2\bfXi_1\big) \la\bfxi_{\bfx_n}\ra +
    \big(\wbo\bfXi_1 + \bfXi_2\wao + \bfXi_2\bfXi_2\big)\bfx_n\Big\rangle
    = 0~.
  \end{align}
  Which leaves us with the third term, which we can write as
  \begin{equation} \label{app:eq:btb-0}
    \frac{1}{2P}\sum_{n=1}^P \la \bfb^T\bfb \ra = \frac{1}{2P}\sum_{n=1}^P \Big(\big\langle\bfxi_{\bfx_n}^T(...)\bfxi_{\bfx_n}^{}\big\rangle + \big\langle\bfxi_{\bfx_n}^T(...)\bfx_n^{}\big\rangle + \big\langle\bfx_n^T(...)\bfxi_{\bfx_n}^{}\big\rangle + \bfx_n^T\big\langle(...)\big\rangle\bfx_n^{}\Big).
  \end{equation}
  In the following, we solve each of the four summands independently.
  We begin with
  \begin{equation}
    \begin{aligned}
      \bigg\langle\bfxi_{\bfx_n}^T\bigg(\wao^T\wbo^T\wbo^{}\wao^{} + \wao^T\wbo^T\wbo^{}\Big\langle\bfXi_1^{}\Big\rangle                                      & + \wao^T\wbo^T\Big\langle\bfXi_2\Big\rangle\wao + \wao^T\wbo^T\Big\langle\bfXi_2\Big\rangle\Big\langle\bfXi_1\Big\rangle                                              \\
      \ +\ \Big\langle\bfXi_1^T\Big\rangle\wbo^T\wbo\wao + \Big\langle\bfXi_1^T\wbo^T\wbo\bfXi_1\Big\rangle                                                   & + \Big\langle\bfXi_1^T\Big\rangle\wbo^T\Big\langle\bfXi_2\Big\rangle\wao + \Big\langle\bfXi_1^T\wbo^T\Big\langle\bfXi_2\Big\rangle\bfXi_1\Big\rangle                  \\
      \ +\ \wao^T\Big\langle\bfXi_2^T\Big\rangle\wbo\wao + \wao^T\Big\langle\bfXi_2^T\big\rangle\wbo\Big\langle\bfXi_1\Big\rangle                             & + \wao^T\Big\langle\bfXi_2^T\bfXi_2\Big\rangle\wao + \wao^T\Big\langle\bfXi_2^T\bfXi_2\Big\rangle\Big\langle\bfXi_1\Big\rangle                                        \\
      \ +\ \Big\langle\bfXi_1^T\Big\rangle\Big\langle\bfXi_2^T\Big\rangle\wbo\wao + \Big\langle\bfXi_1^T\Big\langle\bfXi_2^T\Big\rangle\wbo\bfXi_1\Big\rangle & + \Big\langle\bfXi_1^T\Big\rangle\Big\langle\bfXi_2^T\bfXi_2\Big\rangle\wao + \Big\langle\bfXi_1^T\bfXi_2^T\bfXi_2\bfXi_1\Big\rangle\bigg)\bfxi_{\bfx_n}\bigg\rangle,
    \end{aligned}
  \end{equation}
  where we used that noise is \ac{iid}. Using that noise is zero-centred,
  \erin{needs reindexing}
  we note that summands 2-5, 7-10, and 12-15 resolve to $0$,
  which leaves us with summands 1, 6, 11, and 16 which we solve using \cref{app:eq:random-vector-inner-and-outer-product,app:eq:xi1T-xi1,app:eq:xi1-xi1T,app:eq:xi1-b-xi1}, as
  \begin{equation}
    \begin{aligned}
      \Big\langle\bfxi_{\bfx_n}^T\wao^T\wbo^T\wbo^{}\wao^{}\bfxi_{\bfx_n}\Big\rangle & = \tr\left(\wao^T\wbo^T\wbo^{}\wao^{}\la\bfxi_{\bfx_n}\bfxi_{\bfx_n}^T\ra\right) \\
                                                                                     & = \sigma^2_{\bfx}||\wbo\wao||_F^2,
    \end{aligned}
  \end{equation}

  \begin{equation}
    \begin{aligned}
      \Big\langle\bfxi_{\bfx_n}^T\la\bfXi_1^T\wbo^T\wbo\bfXi_1\ra\bfxi_{\bfx_n}\Big\rangle & = \tr\left(\la\bfXi_1^T\wbo^T\wbo\bfXi_1\ra \la\bfxi_{\bfx_n}\bfxi_{\bfx_n}^T\ra\right) \\
                                                                                           & = \sigma^2_{\bfx} \tr\left(\wbo^T\wbo\la\bfXi_1\bfXi_1^T\ra\right)                      \\
                                                                                           & = \sigma^2_{\bfx}\sigma^2_1N_i||\wbo||_F^2,
    \end{aligned}
  \end{equation}

  \begin{equation}
    \begin{aligned}
      \Big\langle\bfxi_{\bfx_n}^T\wao^T\la\bfXi_2^T\bfXi_2\ra\wao\bfxi_{\bfx_n}\Big\rangle & =  \tr\left(\wao^T\la\bfXi_2^T\bfXi_2\ra\wao\la\bfxi_{\bfx_n}\bfxi_{\bfx_n}^T\ra\right) \\
                                                                                           & = \sigma^2_{\bfx} \tr\left(\wao\wao^T\la\bfXi_2^T\bfXi_2\ra\right)                      \\
                                                                                           & = \sigma^2_{\bfx}\sigma^2_2N_o||\wao||_F^2,
    \end{aligned}
  \end{equation}
  and
  \begin{equation}
    \begin{aligned}
      \la\bfxi_{\bfx_n}^T\bfXi_1^T\bfXi_2^T\bfXi_2\bfXi_1\bfxi_{\bfx_n}\ra & = \tr\left(\la\bfXi_1^T\bfXi_2^T\bfXi_2\bfXi_1\ra\la\bfxi_{\bfx_n}\bfxi_{\bfx_n}^T\ra\right) \\
                                                                           & = \sigma^2_{\bfx}\tr\left(\la\bfXi_1\bfXi_1^T\ra\la\bfXi_2^T\bfXi_2\ra\right)                \\
                                                                           & =\sigma^2_{\bfx}\sigma_1^2N_i\sigma^2_2N_o\tr\left(\bfI_{N_h}\right)                         \\
                                                                           & = \sigma^2_{\bfx}\sigma_1^2N_i\sigma^2_2N_oN_h.
    \end{aligned}
  \end{equation}

  We continue by noting that
  \begin{equation}
    \big\langle\bfxi_{\bfx_n}^T(...)\bfx_n^{}\big\rangle = \tr\left(\big\langle(...)\big\rangle\bfx_n^{}\big\langle\bfxi_{\bfx_n}^T\big\rangle\right) = 0
  \end{equation}
  and
  \begin{equation}
    \big\langle\bfx_n^T(...)\bfxi_{\bfx_n}\big\rangle = \tr\left(\big\langle(...)\big\rangle\big\langle\bfxi_{\bfx_n}^{}\big\rangle\bfx_n^T\right) = 0.
  \end{equation}
  Finally, we solve
  \begin{equation}
    \begin{aligned}\label{app:eq:xt-btb-x}
      \bfx_n^T\bigg( & \la\bfXi_1^T\wbo^T\wbo^{}\bfXi_1^{}\ra + \la\bfXi_1^T\ra\wbo^T\Big\langle\bfXi_2^{}\Big\rangle\wao^{} + \la\bfXi_1^T\wbo^T\Big\langle\bfXi_2\Big\rangle\bfXi_1\ra + \wao^T\la\bfXi_2^T\ra\wbo\Big\langle\bfXi_1\Big\rangle + \wao^T\la\bfXi_2^T\bfXi_2\ra\wao \\
                     & \ +\ \wao^T\la\bfXi_2^T\bfXi_2\ra\Big\langle\bfXi_1\Big\rangle +\la\bfXi_1^T\la\bfXi_2^T\ra\wbo\bfXi_1\ra + \la\bfXi_1^T\ra\la\bfXi_2^T\bfXi_2\ra\wao + \la\bfXi_1^T\bfXi_2^T\bfXi_2\bfXi_1\ra\bigg)\bfx_n,
    \end{aligned}
  \end{equation}
  where we again used that that noise is \ac{iid}.
  Using that noise is zero-centred, we note that summands 2-4, and 6-8 are equal to $0$.
  Which leaves us with summands 1, 5, and 9, which we solve using \cref{app:eq:xi1-b-xi1,app:eq:xi1T-xi1} as
  \begin{equation}
    \begin{aligned}
      \bfx_n^T\la\bfXi_1^T\wbo^T\wbo\bfXi_1\ra\bfx_n & = \sigma_1^2\tr(\wbo^T\wbo)\bfx_n^T\bfI\bfx_n \\
                                                     & = \sigma_1^2||\wbo||_F^2\tr(\bfx_n\bfx_n^T),
    \end{aligned}
  \end{equation}

  \begin{equation}
    \bfx_n^T\wao^T\la\bfXi_2^T\bfXi_2\ra\wao\bfx = N_o\sigma_2^2\tr(\wao^T\wao\bfx_n\bfx_n^T),
  \end{equation}
  and
  \begin{equation}
    \begin{aligned}
      \bfx_n^T\la\bfXi_1^T\bfXi_2^T\bfXi_2\bfXi_1\ra\bfx & = \sigma_1^2\tr\left(\la\bfXi_2^T\bfXi_2\ra\right)\bfx_n^T\bfx_n \\
                                                         & = \sigma_1^2N_o\sigma_2^2\tr(\bfI)\tr(\bfx_n\bfx_n^T)            \\
                                                         & = \sigma_1^2N_o\sigma_2^2N_h\tr(\bfx_n\bfx_n^T).
    \end{aligned}
  \end{equation}
  Resubstitution into \cref{app:eq:btb-0} then yields
  \begin{equation}
    \begin{aligned}
      \label{app:eq:noise-btb-cc}
      \frac{1}{2P}\sum_{n=1}^P \la \bfb^T\bfb \ra = & \frac{\sigma^2_{\bfx}}{2}\big(||\wbo\wao||_F^2 + \sigma^2_1N_i||\wbo||_F^2 + \sigma^2_2N_o||\wao||_F^2 + \sigma_1^2N_i\sigma^2_2N_oN_h\big) \\
      +\                                            & \frac{1}{2P}\big(\sigma_1^2||\wbo||_F^2||\bfX||_F^2 + N_o\sigma_2^2||\wao\bfX||_F^2 + N_h\sigma_1^2N_o\sigma_2^2||\bfX||_F^2\big),
    \end{aligned}
  \end{equation}
  Substituting \cref{app:eq:noise-ata,app:eq:noise-atb,app:eq:noise-btb-cc} back into \cref{the:parameter-and-input-noise} concludes the proof.
\end{proof}

\begin{proof}[Proof of \cref{the:input-noise}]
  Let $\bfxi_n \in \mathcal{R}^{N_i}$ denote a random vector with independent and identically distributed, zero-centred entries with variance $\sigma_x^2$. Then the expected mean squared error over the training data when adding such a random perturbation to each input is
  \begin{equation}
    \begin{aligned}
      \label{app:eq:input-noise}
      \la\frac{1}{2P}\sum_{n=1}^P||\wbo\wao(\bfx_n + \bfxi_n) - \bfy_n||_2^2\ra & = \frac{1}{2P}\sum_{n=1}^P\la||\underbrace{\wbo\wao\bfx_n - \bfy_n}_a + \underbrace{\wbo\wao\bfxi_n}_b||_2^2\ra                     \\
                                                                                & =  \frac{1}{2P}\sum_{n=1}^P \bfa^T\bfa + \frac{1}{2P}\sum_{n=1}^P2 \bfa^T\la\bfb \ra + \frac{1}{2P}\sum_{n=1}^P \la \bfb^T\bfb \ra.
    \end{aligned}
  \end{equation}
  The first summand is identical to \cref{app:eq:noise-ata}, the second summand is
  \begin{equation} \label{app:eq:noise-atb-2}
    \begin{aligned}
      \frac{1}{2P}\sum_{n=1}^P2\bfa^T\left\langle\bfb \right\rangle = \frac{1}{2P}\sum_{n=1}^P2(\wbo\wao\bfx_n -\bfy_n)^T\wbo\wao\left\langle\xi_n\right\rangle \\
       & \ = 0
    \end{aligned}
  \end{equation}
  and for the third summand, using \cref{app:eq:random-vector-inner-and-outer-product,app:eq:xi1T-xi1} we get
  \begin{align}\label{app:eq:noise-btb-2}
    \frac{1}{2P}\sum_{n=1}^P \la \bfb^T\bfb \ra & = \frac{1}{2P}\sum_{n=1}^P\la \bfxi_n^T\wao^T\wbo^T\wbo\wao\bfxi_n \ra \nonumber                \\
                                                & = \frac{1}{2P}\sum_{n=1}^P\tr\left(\wao^T\wbo^T\wbo\wao\la \bfxi_n\bfxi_n^T\ra\right) \nonumber \\
                                                & = \frac{1}{2P} \sum_{n=1}^P \sigma_\bfx^2 \tr\left(\wao^T\wbo^T\wbo\wao\right)                  \\
                                                & = \frac{\sigma_\bfx^2}{2}||\wbo\wao||_F^2 \nonumber.
  \end{align}
  Substituting \cref{app:eq:noise-ata,app:eq:noise-atb-2,app:eq:noise-btb-2} back into \cref{app:eq:input-noise}, then concludes the proof.
\end{proof}

\begin{proof} [Proof of \cref{the:parameter-noise}]
  The expected mean squared error over the training data of a linear solution (\cref{def:general-linear-solution}) under additive, independent and identically distributed zero-centred parameter noise with variance $\sigma_1^2$ and $\sigma_2^2$ is
  \begin{align}
    \label{app:eq:expected-loss-parameter-noise}
    \left\langle \frac{1}{2P}\sum_{n=1}^{P}||\left(\wbo + \bfXi_2\right)\left(\wao + \bfXi_1\right)\bfx_n - \bfy_n||^2_2 \right\rangle & = \frac{1}{2P}\sum_{n=1}^P\bigg\langle||\underbrace{\wbo\wao\bfx_n - \bfy_n}_a + \underbrace{\wbo\bfXi_1\bfx_n +\bfXi_2\wao\bfx_n + \bfXi_2\bfXi_1\bfx_n}_b||^2_2 \bigg\rangle \nonumber \\
                                                                                                                                       & = \frac{1}{2P}\sum_{n=1}^P \bfa^T\bfa + \frac{1}{2P}\sum_{n=1}^P2 \bfa^T\left\langle\bfb \right\rangle + \frac{1}{2P}\sum_{n=1}^P \left\langle \bfb^T\bfb \right\rangle~.
  \end{align}
  The first summand is again identical to \cref{app:eq:noise-ata} and the second summand is
  \begin{equation} \label{app:eq:parameter-noise-atb}
    \frac{1}{2P}\sum_{n=1}^P2\bfa^T\left\langle\bfb \right\rangle = 0
  \end{equation}
  since
  \begin{equation}
    \begin{aligned}
      \left\langle\bfb \right\rangle & = \left\langle\wbo\bfXi_1\bfx_n + \bfXi_2\wao\bfx_n + \bfXi_2\bfXi_1\bfx_n \right\rangle                                                        \\
                                     & = \wbo\left\langle\bfXi_1\right\rangle\bfx_n + \left\langle\bfXi_2\right\rangle\wao\bfx_n + \left\langle\bfXi_2\bfXi_1\right\rangle\bfx_n = 0~.
    \end{aligned}
  \end{equation}
  Which leaves us with the third term, which is identical to \cref{app:eq:xt-btb-x} which equals
  \begin{equation}
    \begin{aligned}
      \label{app:eq:parameter-noise-btb}
      \frac{1}{2P}\sum_{n=1}^P \left\langle \bfb^T\bfb \right\rangle & = \frac{1}{2P}\sum_{n=1}^P \bfx_n^T\big(\sigma_1^2||\wbo||_F^2\bfI + N_o\sigma_2^2\wao^T\wao + N_h\sigma_1^2N_o\sigma_2^2\bfI\big)\bfx_n \\
                                                                     & = \frac{1}{2}\big(\sigma_1^2||\wbo||_F^2\tr(\sigxx) + N_o\sigma_2^2\tr(\wao^T\wao\sigxx) + N_h\sigma_1^2N_o\sigma_2^2\tr(\sigxx)\big)~.
    \end{aligned}
  \end{equation}
  Substituting \cref{app:eq:noise-ata,app:eq:parameter-noise-atb,app:eq:parameter-noise-btb} back into \cref{app:eq:expected-loss-parameter-noise}, then concludes the proof.
\end{proof}

\end{document}